\documentclass{article}
\input{macros-arXiv.sty}

\usepackage{url}            
\usepackage{booktabs}       
\usepackage{amsfonts, bm, amssymb}       
\usepackage{nicefrac}       
\usepackage{microtype}      
\usepackage{tikz} 
\usepackage{float}

\usepackage{array}
\usepackage{diagbox}
\usepackage{mathtools}
\usepackage{multirow, makecell}
\usepackage{rotating}
\usepackage{graphicx}
\usepackage{subfigure}
\usepackage{mathdots}

\usepackage{float}
\usepackage{caption}
\usepackage{setspace}
\usepackage[letterpaper, left=0.9in, right=0.9in, top=1in, bottom=1in]{geometry}

\usepackage[bottom]{footmisc}
\usepackage[colorlinks,linkcolor={blue!75!black},urlcolor=red,citecolor=ForestGreen]{hyperref}
\usepackage{footnotebackref}

\usepackage{threeparttable}

\usepackage{arydshln}

\usepackage{enumerate}
\usepackage{enumitem} 

\usepackage{natbib}

\title{\bf Faster Single-loop Algorithms for \\Minimax Optimization without Strong Concavity}

\author{
Junchi Yang \thanks{Department of Computer Science, ETH Z\"urich, Switzerland. \texttt{junchi.yang@inf.ethz.ch}}
\and
Antonio Orvieto \thanks{Department of Computer Science, ETH Z\"urich, Switzerland. \texttt{antonio.orvieto@inf.ethz.ch}}
\and
Aurelien Lucchi
\thanks{Department of Computer Science, ETH Z\"urich, Switzerland. \texttt{aurelien.lucchi@inf.ethz.ch}}
\and
Niao He \thanks{Department of Computer Science, ETH Z\"urich, Switzerland. \texttt{niao.he@inf.ethz.ch}}
}

\date{\vspace{1ex}}

\begin{document}
\maketitle

\begin{abstract}
    Gradient descent ascent (GDA), the simplest single-loop algorithm for nonconvex minimax optimization, is widely used in practical applications such as generative adversarial networks (GANs) and adversarial training. Albeit its desirable simplicity, recent work shows inferior convergence rates of GDA in theory even assuming strong concavity of the objective on one side.   This paper establishes new convergence results for two alternative {single-loop} algorithms -- \emph{alternating GDA} and \emph{smoothed GDA} -- under the mild assumption that the objective satisfies the Polyak-$\L$ojasiewicz (PL) condition about one variable. We prove that, to find an $\epsilon$-stationary point, (i) alternating GDA and its stochastic variant (without mini batch) respectively require $O(\kappa^{2} \epsilon^{-2})$ and $O(\kappa^{4} \epsilon^{-4})$ iterations, while (ii) smoothed GDA and its stochastic variant (without mini batch) respectively require  $O(\kappa \epsilon^{-2})$ and $O(\kappa^{2} \epsilon^{-4})$ iterations. The latter greatly improves over the vanilla GDA and gives the hitherto best known complexity results among single-loop algorithms under similar settings. We further showcase the empirical efficiency of these algorithms in training GANs and robust nonlinear regression. 
\end{abstract}

\section{Introduction}

\begin{table*}
		\centering
		\small
		\caption{Oracle complexities for deterministic NC-PL problems. Here $\tilde{O}(\cdot)$ hides poly-logarithmic factors. $l$: Lipschitz smoothness parameter; $\mu$: PL parameter, $\kappa$: condition number $\frac{l}{\mu}$; $\Delta$: initial gap of the primal function. We measure the stationarity by $\|\nabla \Phi(x)\|$ with $\Phi(x) = \max_y f(x,y)$ and $\|\nabla f(x,y)\|$. Here $^\star$ means the complexity is derived by translating from one stationary measure to the other (see Proposition \ref{prop conversion}). $\diamond$ it recovers the same complexity for AGDA as Appendix D in \citep{yang2020global}}  
		\vspace{1.5em}
		\renewcommand{\arraystretch}{1.4}
		\begin{threeparttable}[b]
			\begin{tabular}{l | c | c | c | c}
				\hline
				\hline
				\multirow{2}{*}{\textbf{Algorithms}} & 
				\multirow{2}{*}{
					\makecell[c]{
						\textbf{Complexity} \vspace{0.1em}\\
						$\|\nabla\Phi(x)\|\leq \epsilon$
					}
				} & 
				\multirow{2}{*}{
					\makecell[c]{
						\textbf{Complexity} \vspace{0.1em}\\
						$\|\nabla f(x,y)\|\leq \epsilon$
					}
				} 
				&\multirow{2}{*}{\textbf{Loops}}  
				&\multirow{2}{*}{\textbf{Additional assumptions}}
				\\
				& &  &
				\\
				\hline 
				\hline
				GDA \citep{lin2020gradient}
				& $O(\kappa^2 \Delta l\epsilon^{-2}) $
				& $O(\kappa^2 \Delta l\epsilon^{-2})^\star $
				& $1$
				& strong concavity in $y$
				\\
				\hline
				Catalyst-EG \citep{zhang2021complexity}
				& $O(\sqrt{\kappa} \Delta l\epsilon^{-2}) $
				& $O(\sqrt{\kappa} \Delta l\epsilon^{-2})^\star $
				& $3$
				& strong concavity in $y$
				\\
				\hline
				Multi-GDA \citep{nouiehed2019solving}
				& $\tilde{O}(\kappa^3 \Delta l\epsilon^{-2})^\star$
				& $\tilde{O}(\kappa^2 \Delta l\epsilon^{-2})$
				& $2$
				& 
				\\
				\hline
				\textcolor{blue}{Catalyst-AGDA [Appendix \ref{apdx catalyst}]}
				& \textcolor{blue}{$O(\kappa \Delta l\epsilon^{-2})$}
				& \textcolor{blue}{$O(\kappa \Delta l\epsilon^{-2})$}
				& \textcolor{blue}{$2$}
				& 
				\\
				\hline
				\textcolor{blue}{AGDA}
				& \textcolor{blue}{$O(\kappa^2 \Delta l\epsilon^{-2})^\diamond$}
				& \textcolor{blue}{$O(\kappa^2 \Delta l\epsilon^{-2})$}
				& \textcolor{blue}{$1$}
				& 
				\\
				\hline
				\textcolor{blue}{Smoothed-AGDA}
				& \textcolor{blue}{$O(\kappa \Delta l\epsilon^{-2}) $}
				& \textcolor{blue}{$O(\kappa \Delta l\epsilon^{-2}) $}
				& \textcolor{blue}{$1$}
				& 
				\\
				\hline
				\hline
			\end{tabular}
		\end{threeparttable}
		\label{table:summary_results1}
	\end{table*}

Minimax optimization plays an important role in classical game theory and a wide spectrum of emerging machine learning applications, including but not limited to, 
generative adversarial networks (GANs) \citep{goodfellow2014generative}, multi-agent reinforcement learning \citep{zhang2021multi}, and adversarial training \citep{goodfellow2014explaining}.
Many of the aforementioned problems lie outside of the canonical convex-concave setting and can be intractable~\citep{hsieh2021limits,Daskalakis2021TheCO}. Notably, \citet{Daskalakis2021TheCO} showed that, in the worst-case, first-order algorithms need an exponential number of queries to find approximate local solutions for some smooth minimax objectives.

In this paper,  we consider finding stationary points for the general nonconvex smooth minimax optimization problems:
\begin{equation} \label{objective}
    \min _{x \in \mathbb{R}^{d_{1}}} \max _{y \in \mathbb{R}^{d_{2}}} f(x, y) \triangleq \mathbb{E}[F(x, y ; \xi)],
\end{equation}
where $\xi$ is a random vector with support $\Xi$ and $f(x,y)$ is nonconvex in $x$ for any fixed $y$ and possibly nonconcave in $y$. 

Due to its simplicity and single-loop nature, gradient descent ascent (GDA) and its stochastic variants, have become the \textit{de facto} algorithms for training GANs and many other applications in practice.  Their theoretical properties have also been extensively studied in recent literature \citep{lei2020sgd, nagarajan2017gradient, heusel2017gans, mescheder2017numerics, mescheder2018training}. 


\citet{lin2020gradient} provided the complexity results for simultaneous GDA, with simultaneous update for $x$ and $y$, and stochastic GDA (hereafter Stoc-GDA) in finding stationary points when the objective is concave in $y$. In particular, they show that GDA requires $O(\epsilon^{-6})$ iterations and Stoc-GDA without mini-batch requires $O(\epsilon^{-8})$ samples to achieve an $\epsilon$-approximate stationary point. When the objective is strongly concave in $y$, the iteration complexity of GDA can be significantly improved to $O(\kappa^2\epsilon^{-2})$ while the sample complexity for Stoc-GDA reduces to $O(\kappa^3\epsilon^{-4})$ with the large batch of size $O(\epsilon^{-2})$ or $O(\kappa^3\epsilon^{-5})$ without batch, i.e., using a single sample to construct the gradient estimator. Here $\kappa$ is the underlying condition number. However, the following question is still unsettled: can stochastic GDA-type algorithm  achieve the better sample complexity of $O(\epsilon^{-4})$ without large batch size?


Besides the dependence on $\epsilon$, the condition number also plays a crucial role in the convergence rate. There is a long line of research aiming to reduce such a dependency, see e.g. \citep{lin2020near, zhang2021complexity} for some recent results for minimax optimization. These algorithms are typically more complicated as they rely on multiple loops, and are equipped with several acceleration mechanisms. Single-loop algorithms are far more favorable in practice because of their simplicity in implementation. Recently, there are few single-loop variants of GDA, including Alternating Gradient Projection (AGP) \citep{xu2020unified}, Smoothed-AGDA \citep{zhang2020single}. Unfortunately, most of them fail to provide faster convergence in terms of condition number and discuss the stochastic setting even when the strong convexity holds. The question is open: is it possible to improve the dependence on the condition number without resorting to multi-loop procedures? 

In one word, there is urgent need to have \textit{faster convergence in both target accuracy $\epsilon$ and condition number $\kappa$ with single-loop algorithms}. This is even more challenging when the objective is not strongly-concave about $y$.




In this paper, we investigate two viable single-loop algorithms: (i) \emph{alternating GDA} (hereafter AGDA and Stoc-AGDA for their stochastic variance) and (ii) \emph{Smoothed-AGDA}. AGDA, with sequential updates between $x$ and $y$, is one of the most popular algorithms in practice and has an edge over GDA in several settings \citep{zhang2021don}. Smoothed-AGDA, first introduced by \citep{zhang2020single}, utilizes a regularization term to stabilize the performance of GDA when the objective is convex in $y$. We show that these two algorithms can satisfy our need to achieve faster convergence under milder assumptions


We are interested in analyzing their theoretical behaviors under the general \emph{NC-PL setting}, namely, the objective is nonconvex in $x$ and satisfies the Polyak-$\L$ojasiewicz (PL) condition in $y$ \citep{polyak1963gradient}. This is a milder assumption than strong concavity and does not even require the objective to be concave in $y$. Such assumption has been shown to hold in  linear quadratic regulators \citep{fazel2018global}, as well as overparametrized neural networks \citep{liu2020loss}. This setting has driven a lot of the recent progress in the quest for understanding deep neural networks~\citep{lee2017deep, jacot2018neural}, and it therefore appears as an ideal candidate to deepen our understanding of the convergence properties of minimax optimization.

		\begin{table*}
		\centering
		\small
		\caption{Sample complexities for stochastic NC-PL problems when the target accuracy $\epsilon$ is small, i.e. $\epsilon \leq \tilde{O}(\sqrt{\Delta l/\kappa^3})$. We measure the stationarity by $\|\nabla \Phi(x)\|$ with $\Phi(x) = \max_y f(x,y)$ and $\|\nabla f(x,y)\|$. Here $^\star$ means the complexity is derived by translating from one stationary measure to the other (see Proposition \ref{prop conversion}). $^\bigtriangledown$ It assumes the function $f$ is Lipschitz continuous about $x$ and its Hessian is Lipschitz continuous.}
		\vspace{1.5em}
		\renewcommand{\arraystretch}{1.38}
		\begin{threeparttable}[b]
			\begin{tabular}{l | c | c | c |c }
				\hline
				\hline
				\multirow{2}{*}{\textbf{Algorithms}} &
				\multirow{2}{*}{
					\makecell[c]{
						\textbf{Complexity} \vspace{0.1em}\\
						$\|\nabla\Phi(x)\|\leq \epsilon$
					}
				} & 
				\multirow{2}{*}{
					\makecell[c]{
						\textbf{Complexity} \vspace{0.1em}\\
						$\|\nabla f(x,y)\|\leq \epsilon$
					}
				}  
				& \multirow{2}{*}{\textbf{Batch size}} & \multirow{2}{*}{\textbf{Additional assumptions}}
				\\
				& & &
				\\
				\hline 
				\hline
				Stoc-GDA \citep{lin2020gradient}
				& $O(\kappa^3 \Delta l\epsilon^{-4}) $
				& $O(\kappa^3 \Delta l\epsilon^{-4})^\star $
				& $O(\epsilon^{-2})$
				& strong concavity in $y$
				\\
				\hline
				Stoc-GDA \citep{lin2020gradient}
				& $O(\kappa^3 \Delta l\epsilon^{-5}) $
				& $O(\kappa^3 \Delta l\epsilon^{-5})^\star  $
				& $O(1)$
				& strong concavity in $y$
				\\
				\hline
				ALSET \citep{chen2021tighter}
				& $O(\kappa^3 \Delta l\epsilon^{-4})$
				& $O(\kappa^3 \Delta l\epsilon^{-4})^\star $
				& $O(1)$
				& strong concavity in $y$, Lipschitz$^\bigtriangledown$
				\\
				\hline
				\textcolor{blue}{Stoc-AGDA}
				& \textcolor{blue}{$O(\kappa^4 \Delta l\epsilon^{-4})$}
				& \textcolor{blue}{$O(\kappa^4 \Delta l\epsilon^{-4})$}
				& \textcolor{blue}{$O(1)$}
				& 
				\\
				\hline
				\textcolor{blue}{Stoc-Smoothed-AGDA}
				& \textcolor{blue}{$O(\kappa^{2} \Delta l\epsilon^{-4}) $}
				& \textcolor{blue}{$O(\kappa^{2} \Delta l\epsilon^{-4}) $}
				& \textcolor{blue}{$O(1)$}
				& 
				\\
				\hline
				\hline
			\end{tabular}
		\end{threeparttable}
		\label{table:summary_results}
	\end{table*}

\subsection{Contributions}

In this work, we study the convergence of AGDA and Smoothed-AGDA in the NC-PL setting. Our goal is to find an approximate stationary point for the objective function $f(\cdot, \cdot)$ and its primal function $\Phi(\cdot) \triangleq \max _{y} f(\cdot, y)$. For each algorithm, we present a \textit{unified} analysis for the deterministic setting, when we have access to exact gradients of (\ref{objective}), and the stochastic setting, when we have access to noisy gradients. We denote the smoothness parameter by $l$, PL parameter by $\mu$, condition number by $\kappa \triangleq \frac{l}{\mu}$ and initial primal function gap $\Phi(x)-\inf_x\Phi(x)$ by $\Delta$. 

\paragraph{Deterministic setting.} We first show that the output from AGDA is an $\epsilon$-stationary point for both the objective function $f$ and primal function $\Phi$ after $O(\kappa^2\Delta l \epsilon^{-2})$ iterations, which recovers the result of primal function stationary convergence in \citep{yang2020global} based on a different analysis. The complexity is optimal in $\epsilon$, since $\Omega(\epsilon^{-2})$ is the lower bound for smooth optimization problems \citep{carmon2020lower}. We further show that Smoothed-AGDA has $O(\kappa\Delta l \epsilon^{-2})$ complexity in finding an $\epsilon$-stationary point of $f$. We can translate this point to an $\epsilon$-stationary point of $\Phi$ after an additional negligible $\tilde{O}(\kappa)$ oracle complexity. This result improves the complexities of existing single-loop algorithms that require the more restrictive assumption of strong-concavity in $y$ (we refer to this class of function as NC-SC). A comparison of our results to existing complexity bounds is summarized in Table \ref{table:summary_results1}. 

\paragraph{Stochastic setting.} We show that Stoc-AGDA achieves a sample complexity of $O(\kappa^4\Delta l\epsilon^{-4})$ for both notions of stationary measures, without having to rely on the $O(\epsilon^{-2})$ batch size and Hessian Lipschitz assumption used in prior work. This is the first convergence result for stochastic NC-PL minimax optimization and is also optimal in terms of the dependency to $\epsilon$. We further show that the stochastic Smoothed-AGDA (Stoc-Smoothed-AGDA) algorithm  achieves the $O(\kappa^2\Delta l \epsilon^{-4})$ sample complexity in finding an $\epsilon$ stationary point of $f$ or $\Phi$ for small $\epsilon$.  This result improves upon the state-of-the-art complexity $O(\kappa^3\Delta l\epsilon^{-4})$ for NC-SC problems, which is a subclass of the NC-PL family. We refer the reader to Table \ref{table:summary_results} for a comparison.

\subsection{Related Work}
\label{related work}

\paragraph{PL conditions in minimax optimization.} In the deterministic NC-PL setting, \citet{yang2020global} and \citet{nouiehed2019solving} show that AGDA and its multi-step variant, which applies multiple updates in $y$ after one update of $x$, can find an approximate stationary point within $O(\kappa^2\epsilon^{-2})$ and $\tilde{O}(\kappa^2\epsilon^{-2})$ iterations, respectively. Recently, \citet{Fiez2021GlobalCT} showed that GDA converges asymptotically to a differential Stackelberg equilibrium and establish a local convergence rate of $O(\epsilon^{-2})$ for deterministic problems. In comparison, our work establishes non-asymptotic convergence to an $\epsilon$-stationary point regardless of the starting point in both deterministic and stochastic settings, and we also focus on reducing the dependence to the condition number. \citet{xie2021federated} consider NC-PL problems in the federated learning setting, showing $O(\epsilon^{-3})$ communication complexity when each client's objective is Lipschitz smooth. Moreover, there is a few work that aims to find global solutions by further imposing PL condition in $x$ \citep{yang2020global, guo2020communication, guo2020fast}.

 \paragraph{NC-SC minimax optimization. } NC-SC problems are a subclass of NC-PL family. In the deterministic setting, GDA-type algorithms has been shown to have $O(\kappa^2\epsilon^{-2})$ iteration complexity \citep{lin2020gradient, xu2020unified,boct2020alternating, lu2020hybrid}. Later, \citet{lin2020near} and \citet{zhang2021complexity} improve this to $\tilde{O}(\sqrt{\kappa}\epsilon^{-2})$ by utilizing proximal point method and Nesterov acceleration. Comparatively, there are much less study in the stochastic setting. Recently, \citet{chen2021tighter} extend their analysis from bilevel optimization to minimax optimization and show $O(\kappa^3\epsilon^{-4})$ sample complexity for an algorithm called ALSET without $O(\epsilon^{-2})$ batch size required in \citep{lin2020gradient}. ALSET reduces to AGDA in minimax optimization when it only does one step of $y$ update in the inner loop. We also refer the reader to the increasing body of bilevel optimization literature; e.g. \citep{guo2021randomized, ji2020provably,hong2020two, chen2021single}. Also, \citet{luo2020stochastic}, \citet{huang2021adagda} and \citet{tran2020hybrid} explore variance reduced algorithms in this setting under the averaged smoothness assumption. Concurrently, \citet{Fiez2021GlobalCT} prove perturbed GDA converges to $\epsilon$–local minimax equilibria with complexities of $\tilde{O}(\epsilon^{-4})$ and $\tilde{O}(\epsilon^{-2})$ in stochastic and deterministic problems, respectively, under additional second-order conditions. Notably, \citet{zhang2021complexity} and \citet{han2021lower} develop a tight lower complexity bound of $\Omega(\sqrt{\kappa}\epsilon^{-2})$ for the deterministic setting, and \citet{li2021complexity} develop the lower complexity bound of $\Omega\left(\sqrt{\kappa} \epsilon^{-2}+\kappa^{1 / 3} \epsilon^{-4}\right)$ for the stochastic setting. Other than first-order algorithms, there are a few explorations of zero-order methods \citep{xu2021zeroth, huang2020accelerated, xu2020gradient, wang2020zeroth, liu2020min,anagnostidis2021direct} and second-order methods \citep{luo2021finding, chen2021escaping}. All the results above hold in the NC-SC regime, while the PL condition is significantly weaker than strong-concavity as it lies in the nonconvex regime. 

\paragraph{Other nonconvex minimax optimization.} There is a line of work focusing on the setting where the objective is (non-strongly) concave about $y$, but achieves slower convergence than NC-SC minimax optimization for both general deterministic and stochastic problems \citep{zhao2020primal, thekumparampil2019efficient, ostrovskii2021efficient, rafique2021weakly}. For nonconvex-nocnoncave (NC-NC) problems, different notions of local optimal solutions as well as their properties have been investigated in~\citep{mangoubi2021greedy, jin2020local, fiez2020local, ratliff2013characterization, ratliff2016characterization}. At the same time, many works have studied the relations between the stable limit points of the algorithms and local solutions \citep{daskalakis2018limit, mazumdar2020gradient}. After the hardness in finding an approximate stationary point has been studied in \citep{Daskalakis2021TheCO, hsieh2021limits, letcher2020impossibility, wang2019solving}, some research works then turned to identifying the conditions required for convergence~\citep{grimmer2020landscape, lu2021s, abernethy2021last}. One of the widely explored conditions among them is the Minty variational inequality (MVI), or some approximate notions  \citep{diakonikolas2021efficient, liu2021first, liu2019towards, malitsky2020golden, mertikopoulos2018optimistic,song2020optimistic, zhou2017stochastic}. Recently, \citet{ostrovskii2021nonconvex} study the nonconvex-nonconcave minimax optimization when the domain of $y$ is small.

\section{Preliminaries}

\paragraph{Notations. } Throughout the paper, we let $\|\cdot\|=\sqrt{\langle\cdot, \cdot\rangle}$ denote the $\ell_2$ (Euclidean) norm and $\langle\cdot, \cdot\rangle$ denote the inner product. For non-negative functions $f(x)$ and $g(x)$, we write $f=O(g)$ if $f(x)\leq cg(x)$ for some $c>0$, and $f=\tilde{O}(g)$ to omit poly-logarithmic terms. We define the primal-dual gap of a function $f(\cdot, \cdot)$ at a point $(\hat{x},\hat{y})$ as $
		\gap_f(\hat{x}, \hat{y}) \triangleq \max_{y \in\mathbb{R}^{d_2}} f(\hat{x},y) - \min_{x \in \mathbb{R}^{d_1}} f(x,\hat{y})$.
		
\vspace{3mm}
We are interested in minimax problems of the form:
\begin{equation} \label{objective2}
    \min _{x \in \mathbb{R}^{d_{1}}} \max _{y \in \mathbb{R}^{d_{2}}} f(x, y) \triangleq \mathbb{E}[F(x, y ; \xi)],
\end{equation}
where $\xi$ is a random vector with support $\Xi$, and $f$ is possibly nonconvex-nonconcave. We now present the main setting considered in this paper.

\begin{assume} [Lipschitz Smooth]
The function $f$ is differentiable and there exists a positive constant $l$ such that
{\small
\begin{align*}
     &\left\| \nabla _ { x } f \left( x _ { 1 } , y _ {1} \right) - \nabla _ { x } f \left( x _ { 2 } , y _ {2} \right) \right\| \leq l [\left\| x _ { 1 } - x _ { 2 } \right\| + \left\| y _ { 1 } - y _ { 2 } \right\|],\\
    &\left\| \nabla _ { y } f \left( x _{1} , y _ { 1 } \right) - \nabla _ { y } f \left( x_ {2} , y _ { 2 } \right) \right\| \leq l [\left\| x _ { 1 } - x _ { 2 } \right\| + \left\| y _ { 1 } - y _ { 2 } \right\|],
\end{align*} \label{Lipscthitz gradient}
}%
holds for all $x_1, x_2\in \mathbb{R}^{d_1}, y_1$, $y_2 \in \mathbb{R}^{d_2}$.
\end{assume}

\begin{assume} [PL Condition in $y$]
\label{PL assumption}
For any fixed $x$, $\max_{y\in\mathbb{R}^{d_2}} f(x,y)$ has a nonempty solution set and a finite optimal value. There exists $\mu > 0$ such that: $
   \Vert \nabla_yf(x,y) \Vert^2 \geq 2\mu [\max_{y}f(x, y)-f(x,y)], \forall x,y.
$
\end{assume}

The PL condition was originally introduced in \citep{polyak1963gradient} who showed that it guarantees global convergence of gradient descent at a linear rate. This condition is shown in~\citep{karimi2016linear} to be weaker than strong convexity as well as other conditions under which gradient descent converges linearly. The PL condition has also drawn much attention recently as it was shown to hold for various non-convex applications of interest in machine learning~\citep{fazel2018global,cai2019global}, including problems related to deep neural networks \citep{du2019gradient,liu2020loss}. In this work, we assume that the objective function $f$ in (\ref{objective2}) is Lipschitz smooth and satisfies the PL condition about the dual variable $y$, i.e. Assumption \ref{Lipscthitz gradient} and \ref{PL assumption}, which is the same setting as in~\citep{nouiehed2019solving} and \citep{yang2020catalyst} (Appendix D). However, to the best of our knowledge, stochastic algorithms have not yet been studied under such a setting.

From now on, we will define $\Phi(x)\triangleq\max_y f(x,y)$ as the primal function and $\kappa\triangleq \frac{l}{\mu}$ as the condition number. We will assume that $\Phi(\cdot)$ is lower bounded by a finite $\Phi^*$. According to \citep{nouiehed2019solving}, $\Phi(\cdot)$ is $2\kappa l$-lipschitz smooth with Assumption \ref{Lipscthitz gradient} and \ref{PL assumption}. There are two popular and natural notions of stationarity for minimax optimization in the form of (\ref{objective2}): one is measured with $\nabla f$ and the other is measured with $\nabla \Phi$. We give the formal definitions below.
\begin{definition} [Stationarity Measures] \quad
\begin{itemize}[topsep=1pt]
\itemsep0.1em 
    \item[\textbf{a)}] $(\hat{x}, \hat{y})$ is an $(\epsilon_1, \epsilon_2)$-stationary point of a differentiable function $f(\cdot, \cdot)$ if $\|\nabla_x f(\hat{x}, \hat{y}) \| \leq \epsilon_1$ and  $\|\nabla_y f(\hat{x}, \hat{y}) \| \leq \epsilon_2$. If  $(\hat{x}, \hat{y})$ is an $(\epsilon, \epsilon)$-stationary point, we call it $\epsilon$-stationary point for simplicity. 
    \item[\textbf{b)}] $\hat{x}$ is an $\epsilon$-stationary point of a differentiable function $\Phi(\cdot)$ if $\|\nabla\Phi(\hat{x})\|\leq\epsilon$.
\end{itemize}
\end{definition}
These two notions can be translated to each other by the following proposition.

\begin{prop} [Translation between Stationarity Measures] \label{prop conversion} 
    \leavevmode 
    
    \textbf{a)} Under Assumptions \ref{Lipscthitz gradient} and \ref{PL assumption}, if $\hat{x}$ is an $\epsilon$-stationary point of $\Phi$ and $\|\nabla_y f(\hat{x}, \tilde{y})\|\leq \epsilon'$, then we can find another $\hat{y}$ by maximizing $f(\hat{x},\cdot)$ from the initial point $\tilde{y}$ with (stochastic) gradient ascent such that $(\hat{x},\hat{y})$ is an $O(\epsilon)$-stationary point of $f$, which requires $O\left(\kappa \log\left(\frac{\kappa\epsilon'}{\epsilon} \right) \right)$ gradients or $\tilde{O}\left(\kappa + \kappa^3\sigma^2\epsilon^{-2} \right)$ stochastic gradients. 
    
    \textbf{b)} Under Assumptions \ref{Lipscthitz gradient} and \ref{PL assumption}, if $(\tilde{x}, \tilde{y})$ is an $(\epsilon, \epsilon/\sqrt{\kappa})$-stationary point of $f$, then we can find an $O(\epsilon)$-stationary point of $\Phi$ by approximately solving $\min_x\max_y f(x, y) + l\|x-\tilde{x}\|^2$ from the initial point $(\tilde{x}, \tilde{y})$ with (stochastic) AGDA, which requires $O\left(\kappa \log\left(\kappa \right) \right)$ gradients or $\tilde{O}\left(\kappa+ \kappa^5\sigma^2\epsilon^{-2} \right)$ stochastic gradients.
\end{prop}

\begin{remark}
The proposition implies that we can convert an $\epsilon$-stationary point of $\Phi$ to an $\epsilon$-stationary point of $f$ and an $(\epsilon, \epsilon/\sqrt{\kappa})$-stationary point of $f$ to an $\epsilon$-stationary point of $\Phi$, at a low cost in  $1/\epsilon$ dependency compared to the complexity of finding the stationary point of either notion. Therefore, we consider the stationarity of $\Phi$ a slightly stronger notion than the other. \citet{lin2020gradient} establish the similar conversion under the NC-SC setting, but it requires an $(\epsilon/\kappa)$-stationary point of $f$ to find an $\epsilon$-stationary point of $\Phi$. Later we will use this proposition to establish the stationary convergence for some algorithm. 
\end{remark}

Finally, we assume to have access to unbiased stochastic gradients of $f$ with bounded variance.

\begin{assume}[Stochastic Gradients]\label{stochastic gradients} 
$G_x(x,y, \xi)$ and $G_y(x,y, \xi)$ are unbiased stochastic estimators of $\nabla_x f(x, y)$ and $\nabla_y f(x, y)$ and have variances bounded by $\sigma^2>0$.  
\end{assume}

\section{Stochastic AGDA}

\begin{algorithm}[h] 
    \caption{Stoc-AGDA}
    \begin{algorithmic}[1]
        \STATE Input: $(x_0,y_0)$, step sizes $\tau_1>0, \tau_2>0$
        \FORALL{$t = 0,1,2,..., T-1$}
            \STATE Draw two i.i.d. samples $\xi^t_{1}, \xi^t_{2}$ 
            \STATE $x_{t+1}\gets  x_t-\tau_1 G_x(x_t,y_t, \xi^t_{1})$
            \STATE $y_{t+1}\gets y_t+\tau_2 G_y(x_{t+1},y_t, \xi^t_{2})$
        \ENDFOR
        \STATE Output: choose $(\hat{x}, \hat{y})$ uniformly from $\{(x_t, y_t)\}_{t=0}^{T-1}$
    \end{algorithmic} \label{agda random}
\end{algorithm}

Stochastic alternating gradient descent ascent (Stoc-AGDA) presented in Algorithm \ref{agda random} sequentially updates primal and dual variables with simple stochastic gradient descent/ascent. In each iteration, only two samples are drawn to evaluate stochastic gradients. Here $\tau_1$ and $\tau_2$ denote the stepsize of $x$ and $y$, respectively, and they can be very different.

\begin{theorem} \label{thm agda}
Under Assumptions \ref{Lipscthitz gradient}, \ref{PL assumption} and \ref{stochastic gradients}, if we apply Stoc-AGDA with stepsizes $\tau_1 = \min\bigg\{\frac{\sqrt{\Delta}}{4\sigma\kappa^2\sqrt{Tl}}, \frac{1}{68l\kappa^2} \bigg\}$ and $\tau_2 = \min\bigg\{\frac{17\sqrt{\Delta}}{\sigma\sqrt{Tl}}, \frac{1}{l} \bigg\}$, then we have
\begin{align*}
    \frac{1}{T} \sum_{t=0}^{T-1}\mathbb{E}\|\nabla \Phi(x_t)\|^2 \leq  & \frac{1088 l\kappa^2}{T}\Delta + \frac{136l\kappa^2}{T}a_0 + \frac{8\kappa^2\sqrt{l} a_0}{\sqrt{\Delta T}}\sigma + \frac{1232\kappa^2\sqrt{l\Delta}}{\sqrt{T}}\sigma,
\end{align*}
where $\Delta = \Phi(x_0)-\Phi^*$ and $a_0 := \Phi(x_0) - f(x_0, y_0)$. This implies a sample complexity of $O\left(\frac{l\kappa^2\Delta}{\epsilon^2} + \frac{l\kappa^4\Delta\sigma^2}{\epsilon^4} \right)$ to find an $\epsilon$-stationary point of $\Phi$.
\end{theorem}

We can either use Proposition \ref{prop conversion} to translate to the other notion with extra computations or show that Stoc-AGDA directly outputs an $\epsilon$-stationary point of $f$ with the same sample complexity. 

\begin{corollary} \label{coro agda}
Under the same setting as Theorem \ref{thm agda}, the output $(\hat{x}, \hat{y})$ from Stoc-AGDA satisfies $\mathbb{E}\|\nabla_x f(\hat{x}, \hat{y})\| \leq \epsilon$ and $\mathbb{E}\|\nabla_y f(\hat{x}, \hat{y})\| \leq \epsilon$ after $O\left(\frac{l\kappa^2\Delta}{\epsilon^2} + \frac{l\kappa^4\Delta\sigma^2}{\epsilon^4} \right)$ iterations, which implies the same sample complexity as Theorem \ref{thm agda}.
\end{corollary}

\begin{remark}
The dependency on $a_0 = \Phi(x_0) - f(x_0, y_0)$ can be improved by initializing $y_0$ with gradient ascent or stochastic gradient ascent to maximize the function $f(x_0, \cdot)$ satisfying the PL condition, which has exponential convergence in the deterministic setting and $O(\frac{1}{T})$ sublinear rate in the stochastic setting \citep{karimi2016linear}.
\end{remark}

\begin{remark}
The complexity above has different dependency as a function of $\epsilon$ and $\kappa$ for the terms with and without the variance term $\sigma$. When $\sigma = 0$,  iterations the output from AGDA after $O\left(l\kappa^2\Delta\epsilon^{-2}\right)$ will be an $\epsilon$-stationary point of both $f$ and $\Phi$. It recovers the same complexity result in \citep{yang2020catalyst} for the primal function stationary convergence. \citet{nouiehed2019solving} show the same complexity for multi-GDA based on the stationary measure of $f$, which implies $O(l\kappa^3\Delta\epsilon^{-2})$ complexity for the stationary convergence of $\Phi$ by Proposition \ref{prop conversion}. See Table \ref{table:summary_results1} for more comparisons.
\end{remark}

\begin{remark}
When $\sigma>0$, we establish the brand-new sample complexity of $O(l\kappa^4\Delta\epsilon^{-4})$ for Stoc-AGDA. It is the first analysis of stochastic algorithms for NC-PL minimax problems. The dependency on $\epsilon$ is optimal, because the lower complexity bound of $\Omega(\epsilon^{-4})$ for stochastic nonconvex optimization \citep{arjevani2019lower} still holds when considering $f(x,y) = F(x)$ for some nonconvex function $F(x)$. Even under the strictly stronger assumption of imposing strong-concavity in $y$, to the best of our knowledge, it is the first time that vanilla stochastic GDA-type algorithm is showed to achieve $O(\epsilon^{-4})$ sample complexity without either increasing batch size as in \citep{lin2020gradient} or Lipschitz continuity of $f(\cdot, y)$ and its Hessian as in \citep{chen2021tighter}. In \citep{lin2020gradient}, they show a worse complexity of $O(\epsilon^{-5})$ for GDA with $O(1)$ batch size. We refer the reader to Table \ref{table:summary_results}.
\end{remark} 

\begin{remark}
We point out that under our weaker assumption, the dependency on the condition number $\kappa$ is slightly worse than that in \citep{lin2020gradient, chen2021tighter}. 
If only $O(1)$ samples are available in each iteration, Stoc-GDA only achieves $O(\epsilon^{-5})$ sample complexity \citep{lin2020gradient}. On the other hand, the analysis in \citep{chen2021single} is not applicable here. It uses a potential function $V_t =  \Phi(x_t) + O(\mu)\|y_t-y^*(x_t)\|^2]$, where $y^*(x_t) = \argmax_y f(x,y)$. To show a descent lemma for $\mathbb{E}[V_t]$, it shows the Lipschitz smoothness of $y^*(\cdot)$, which heavily depends on Lipschtiz continuity of $f$ and its hessian, while under PL condition $y^*(x)$ might not be unique and we no longer make additional Lipschitz assumptions. Instead, we present an analysis based on the potential function $V_t = \Phi(x_t) + O(1) [\Phi(x_t) - f(x_t, y_t)]$ (see Appendix \ref{apdx agda}).

\end{remark}


\section{Stochastic Smoothed AGDA}

\begin{algorithm}[h] 
    \caption{Stochastic Smoothed-AGDA}
    \begin{algorithmic}[1]
        \STATE Input: $(x_0,y_0,z_0)$, step sizes $\tau_1>0, \tau_2>0$
        \FORALL{$t = 0,1,2,..., T-1$}
            \STATE Draw two i.i.d. samples $\xi^t_{1}, \xi^t_{2}$
            \STATE $x_{t+1} = x_t-\tau_1 [ G_x(x_t, y_t, \xi^t_1)+p(x_t-z_t)]$ \label{agda step 1}
            \STATE $y_{t+1} = y_t+\tau_2 G_y(x_{t+1}, y_t, \xi^t_2)$ \label{agda step 2}
            \STATE $z_{t+1} = z_t + \beta(x_{t+1}-z_t)$ \label{prox point}
        \ENDFOR
        \STATE Output: choose $(\hat{x}, \hat{y})$ uniformly from $\{(x_t, y_t)\}_{t=0}^{T-1}$
    \end{algorithmic} \label{stoc acc AGDA}
\end{algorithm}

Stochastic Smoothed-AGDA presented in Algorithm \ref{stoc acc AGDA} is closely related to proximal point method (PPM) on the primal function $\Phi(\cdot)$. In each iteration, we consider solving an auxiliary problem: $\min_x \Phi(x) + \frac{p}{2}\|x-z_t\|^2$, which is equivalent to:
\begin{equation*}
    \min_x \max_y \hat{f}(x, y; z_t) \triangleq f(x,y) + \frac{p}{2}\|x-z_t\|^2,
\end{equation*}
where $z_t$ is called a proximal center to be defined later. Recently, proximal type algorithms including Catalyst have been shown to efficiently accelerate minimax optimization \citep{lin2020near, yang2020catalyst, zhang2021complexity, luo2021near}. While these algorithms require multiple loops to solve the auxiliary problem to some high accuracy\footnote{In Appendix \ref{apdx catalyst}, we present a two-loop Catalyst algorithm combined with AGDA (Catalyst-AGDA) that achieves the same complexity as Algorithm \ref{stoc acc AGDA} in the deterministic setting. }, Stoc-Smoothed-AGDA only applies one step of Stoc-AGDA to solve it from the point $(x_t, y_t)$ as in step \ref{agda step 1} and \ref{agda step 2}. Step \ref{prox point} in Algorithm \ref{stoc acc AGDA} with some $\beta \in (0,1)$ guarantees that the proximal point $z_t$ in the auxiliary problem is not too far from the previous one $z_{t-1}$. Smoothed-AGDA was first introduced by \citet{zhang2020single} in the deterministic nonconvex-concave minimax optimization. To the best of our knowledge, its convergence has not been  discussed in either the stochastic or the NC-PL setting. 

Stoc-Smoothed-AGDA  still maintains the single-loop structure and use only $O(1)$ samples in each iteration. If we choose $\beta = 1$ or $p = 0$, it reduces to Stoc-AGDA. Later in the analysis, we choose $p = 2l$ so that the auxiliary problem is $l$-strongly convex in $x$. 
We will see in the next theorem that this quadratic regularization term enables Smoothed-AGDA to take larger stepsizes for $x$ compared to AGDA. In Smoothed-AGDA, the ratio between stepsize of $x$ and $y$ is $\Theta(1)$\footnote{In Appendix \ref{apdx catalyst}, we show Catalyst-AGDA takes the stepsizes of the same order in the deterministic setting. }, while this ratio is $\Theta(1/\kappa^2)$ in AGDA.

\begin{theorem} \label{thm s-agda}
Under Assumptions \ref{Lipscthitz gradient}, \ref{PL assumption} and \ref{stochastic gradients}, if we apply Algorithm \ref{stoc acc AGDA} with $\tau_1 = \min\Big\{\frac{\sqrt{\Delta}}{2\sigma\sqrt{Tl}}, \frac{1}{3l} \Big\}$, $\tau_2 = \min\Big\{\frac{\sqrt{\Delta}}{96\sigma\sqrt{Tl}}, \frac{1}{144l} \Big\}$, $p = 2l$ and $\beta = \frac{\tau_2\mu}{1600}$, then
\begin{align*}
    \frac{1}{T}\sum_{t=0}^{T-1}\mathbb{E}\big\{\left\|\nabla_x f\left(x_{t}, y_t\right)\right\|^{2} + \kappa\left\|\nabla_y f\left(x_{t}, y_t\right)\right\|^{2} \big\} \leq \frac{c_0l\kappa}{T}[\Delta + b_0] + 
    \frac{c_1\kappa\sqrt{l} b_0}{\sqrt{\Delta T}}\sigma +  \frac{c_2\kappa\sqrt{l\Delta}}{\sqrt{T}}\sigma &,
\end{align*}
where $\Delta = \Phi(z_0) - \Phi^*$ and $b_0 = 2\gap_{\hat{f}(\cdot, \cdot; z_0)}(x_0, y_0)$ is the primal-dual gap of the first auxiliary function at the initial point, and $c_0$, $c_1$ and $c_2$ are $O(1)$ constants. This implies the sample complexity of $O\left(\frac{l\kappa\Delta}{\epsilon^2} + \frac{l\kappa^2\Delta\sigma^2}{\epsilon^4} \right)$ to find an $(\epsilon, \epsilon/\sqrt{\kappa})$-stationary point of $f$.
\end{theorem}

\begin{remark}
In the theorem above, $b_0$ measures the optimality of $(x_0, y_0)$ in the first auxiliary problem: $\min_x\max_y f(x, y) + l\|x-z_0\|^2$, which is $l$-strongly convex about $x$ and $\mu$-PL about $y$. Therefore, the dependency on $b_0$ can be reduced if we initialize $(x_0, y_0)$ by approximately solving the first auxiliary problem with (Stochastic) AGDA, which converges exponentially in the deterministic setting and sublinearly at $O(1/T)$ rate in the stochastic setting for strongly-convex-PL minimax optimization \citep{yang2020global}. 
\end{remark}

By Proposition \ref{prop conversion}, we can convert the output from Stoc-Smoothed-AGDA to an $O(\epsilon)$-stationary point of $\Phi$.

\begin{corollary} \label{coro s-agda}
From the output $(\hat{x}, \hat{y})$ of stochastic Smoothed-AGDA, we can apply (stochastic) AGDA to find an $O(\epsilon)$-stationary point of $\Phi$ by approximately solving $\min_x\max_y f(x, y) + l\|x-\hat{x}\|^2$. The total complexity is $O\left(\frac{l\kappa\Delta}{\epsilon^2}\right)$ in the deterministic setting and $\Tilde{O}\left(\frac{l\kappa\Delta}{\epsilon^2} + \frac{l\kappa^2\Delta\sigma^2}{\epsilon^4} + \frac{\kappa^5\sigma^2}{\epsilon^2} \right)$ in the stochastic setting. 
\end{corollary}

\begin{remark}
In the deterministic setting, the translation cost is $\kappa\log(\kappa)$, which is dominated by the complexity of finding $(\epsilon, \epsilon/\sqrt{\kappa})$-stationary point of $f$ in Theorem \ref{thm s-agda}. In the stochastic setting, the extra translation cost $\Tilde{O}\left(\frac{\kappa^5\sigma^2}{\epsilon^2}\right)$ is low in the dependency of $\frac{1}{\epsilon}$ but larger in terms of the condition number. In practice, the inverse of the target accuracy is usually large. We leave the question of reducing translation cost and whether Stocastic Smoothed-AGDA can directly output an approximate stationary point of $\Phi$ to future research. 

\end{remark}

\begin{remark}
The term without variance $\sigma$ has better dependency on $\epsilon$ and $\kappa$ than the term with $\sigma$. In the deterministic setting, Smoothed-AGDA achieves the complexity of $O(l\kappa\Delta\epsilon^{-2})$, which improves over AGDA \citep{yang2020global} and Multi-AGDA \citep{nouiehed2019solving} with either notion of stationarity. Notably, this complexity under our weaker assumptions is better than that of other single-loop algorithms under a stronger assumption of strong-concavity in $y$ (see Table \ref{table:summary_results}). Recently, \citet{zhang2021complexity} provide a tight lower bound of $O(l\sqrt{\kappa}\Delta\epsilon^{-2})$ for deterministic NC-SC minimax optimization. However, we do not expect the same complexity can be achieved under weaker assumptions.

\end{remark}

\begin{remark}
In the stochastic setting, we show Stoc-Smoothed-AGDA achieves a sample complexity of $O(l\kappa^2\Delta\epsilon^{-4})$ for finding an $\epsilon$-stationary point of $f$. To find an $\epsilon$-stationary point of $\Phi$, it bears an additional complexity of $O(\kappa^5\sigma^2\epsilon^{-2})$, which is negligible as long as $\epsilon$ is asymptotically small, i.e. when $\epsilon \leq \Tilde{O}(\sqrt{\Delta / l\kappa^3})$. This sample complexity improves over $O(l\kappa^4\Delta\epsilon^{-4})$ sample complexity of Stoc-AGDA in NC-PL setting, and even $O(l\kappa^3\Delta\epsilon^{-4})$ complexity of Stoc-GDA \citep{lin2020gradient} and ALSET \citep{chen2021tighter} in NC-SC setting. Moreover, this sample complexity improvement comes without any large batch size, additional Lipschitz assumptions, or multi-loop structure. Very recently, \citet{li2021complexity} develop the lower complexity bound of $\Omega\left(\sqrt{\kappa} \epsilon^{-2}+\kappa^{1 / 3} \epsilon^{-4}\right)$ in NC-SC setting, but there is no matching upper bound yet. 
\end{remark}

\section{Experiments}

\begin{figure*}[t]
    \centering
    \begin{minipage}[b]{0.3\linewidth}
    \includegraphics[width=0.94\textwidth]{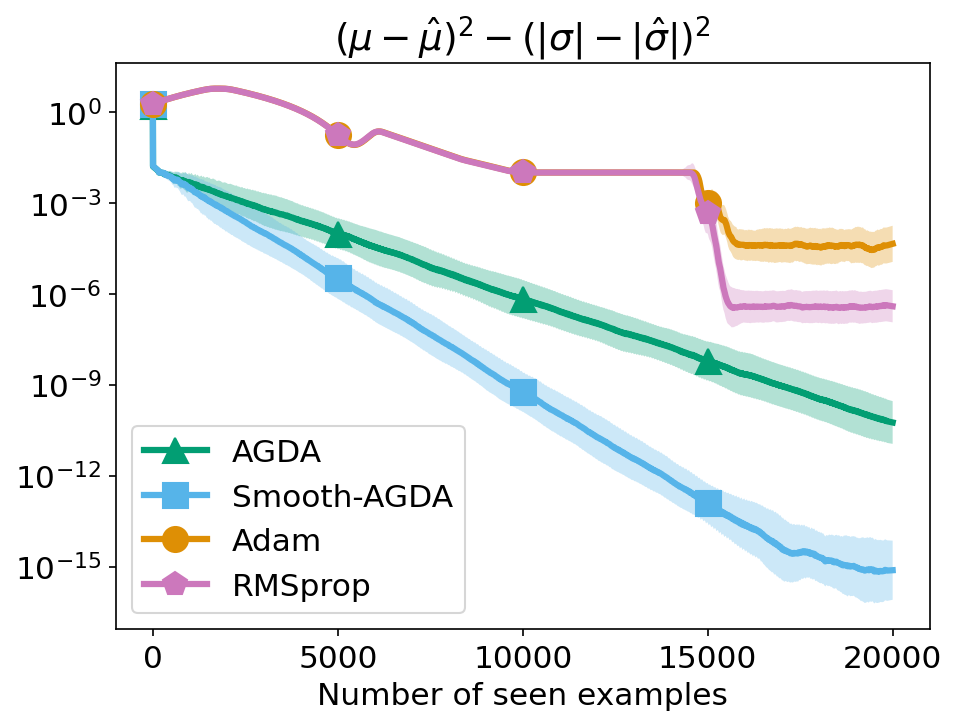}
    \end{minipage}
    \begin{minipage}[b]{0.3\linewidth}
    \includegraphics[width=0.94\textwidth]{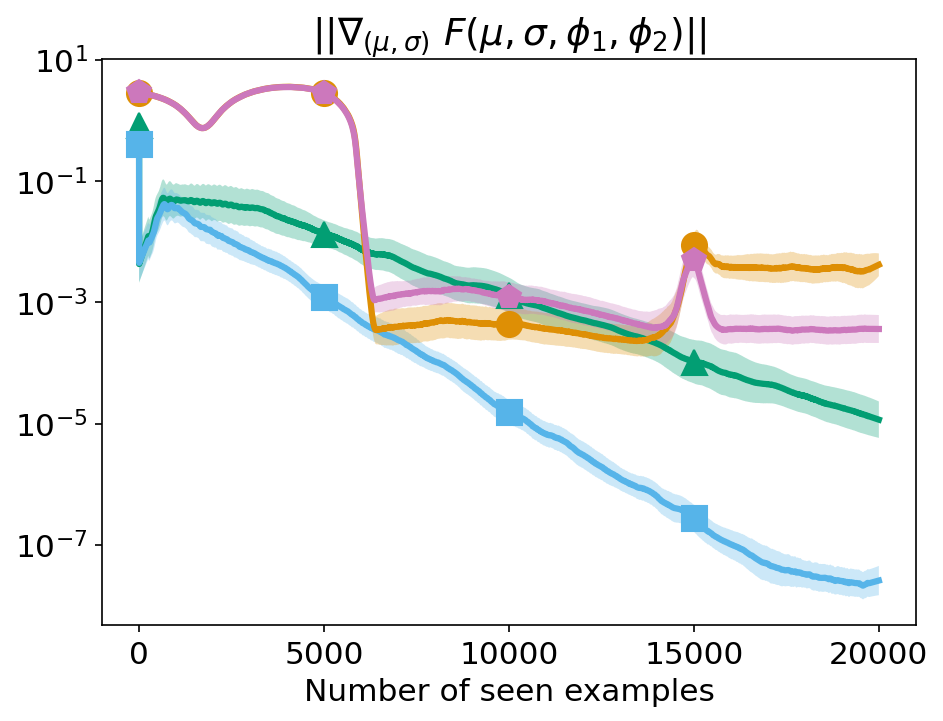}
    \end{minipage}
    \begin{minipage}[b]{0.3\linewidth}
    \includegraphics[width=0.94\textwidth]{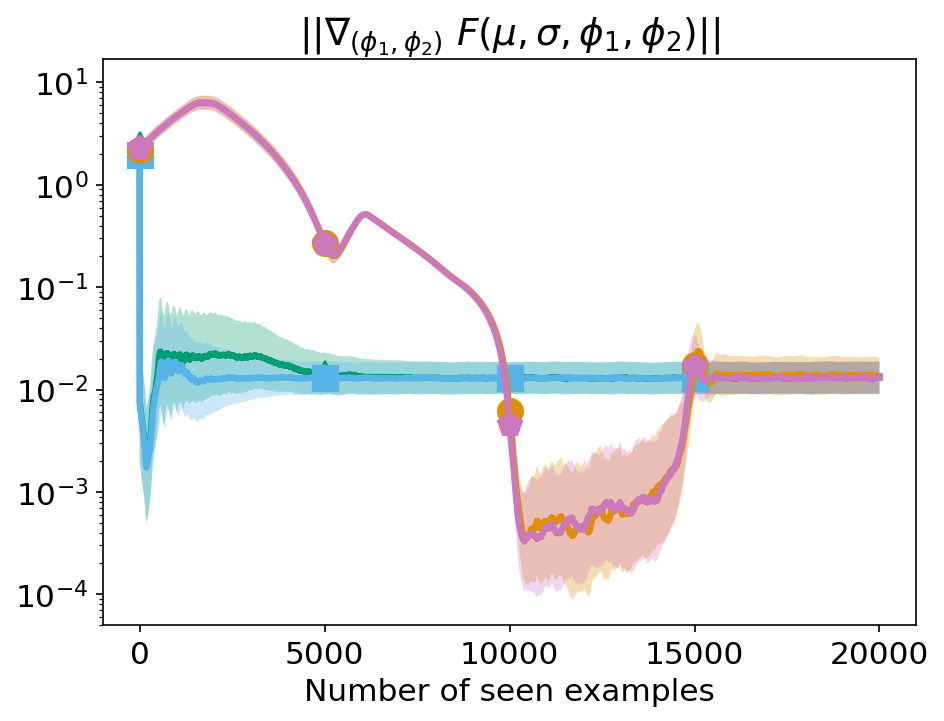}
    \end{minipage}
    \caption{\small Training of  a toy regularized WGAN with linear generator. Shown is the evolution of the \textit{stochastic} gradients norm and the distance to the optimum. All methods are tuned at best for a minibatch size of $100$, and each experiment is repeated 5 times~(1 std shown). For Adam and RMSprop, we tuned over 4 learning rates~($1e-4, 5e-4, 1e-3, 5e-3$) and 2 momentum parameters $0.5, 0.9$. The optimal configuration is obtained for a stepsize of $5e-4$ and momentum $0.5$.  For stochastic AGDA we considered each combination of $\tau_1,\tau_2\in\{1e-2,5e-2,1e-1,5e-1,1\}$. The optimal configuration was found to be $\tau_1 = 5e-1,\tau_2=1$. For stochastic Smoothed-AGDA we use $\beta = 0.9$, $p=10$ and tuned it to best: $\tau_1 = 5e-1,\tau_2=5e-1$.}
    \label{fig:WGAN_res}
    \vspace{-3mm}
\end{figure*}

We illustrate the effectiveness of stochastic AGDA~(Algorithm~\ref{agda random}) and stochastic Smoothed-AGDA~(Algorithm~\ref{stoc acc AGDA}) for solving NC-PL min-max problems. In particular, we show that the smoothed version of stochastic AGDA can compete with state-of-the-art deep learning optimizers~\footnote{Code available at~\url{https://github.com/aorvieto/NCPL.git}}.

\paragraph{Toy WGAN with linear generator.} We consider the same setting as~\citep{loizou2020stochastic}, i.e. using a Wasserstein GAN~\citep{arjovsky2017wasserstein} to approximate a one-dimensional Gaussian distribution. In particular,  we have a dataset of real data $x^{real}$ and
latent variable $z$ from a normal distribution with mean $0$ and variance $1$. The generator is defined as
$G_{\mu,\sigma}(z) = \mu + \sigma z$ and the discriminator~(a.k.a the critic) as $D_{\phi}(x) =\phi_1 x + \phi_2 x^2$, where $x$ is either real data or fake data from the generator. The true data is generated from $\hat\mu=0,\hat\sigma = 0.1$. 
The problem can be written in the form of:
\begin{align*}
    \min_{\mu,\sigma}\max_{\phi_1,\phi_2} & \ f(\mu, \sigma, \phi_1, \phi_2)  \triangleq \mathbb{E}_{(x^{real},z)\sim \mathcal{D}}\ D_{\phi}(x^{real})-D_{\phi}(G_{\mu,\sigma}(z)) - \lambda \|\phi\|^2,
    \vspace{-2mm}
\end{align*}
where $\mathcal{D}$ is the distribution for the real data and latent variable, and the regularization $\lambda\|\phi\|^2$ with $\lambda = 0.001$ makes the problem strongly concave. This problem is non-convex in $\sigma$: indeed since $z$ is symmetric around zero, both $\sigma$ and $-\sigma$ are solutions. We fixed the batch size to 100 and tuned each algorithm at best~(see plots in the appendix). Each experiment is repeated for 3 times.  In Figure~\ref{fig:WGAN_res} we provide evidence of the superiority of Stoc-Smoothed-AGDA over Stoc-AGDA, Adam~\citep{kingma2014adam} and RMSprop~\citep{tieleman2012lecture}. As the reader can notice, Stoc-Smoothed-AGDA is competitive with fine-tuned popular adaptive methods, and provides a significant speedup over AGDA with carefully tuned learning rates, which verifies our theoretical results.

\begin{figure}[t]
    \centering
    \includegraphics[width=0.3\textwidth]{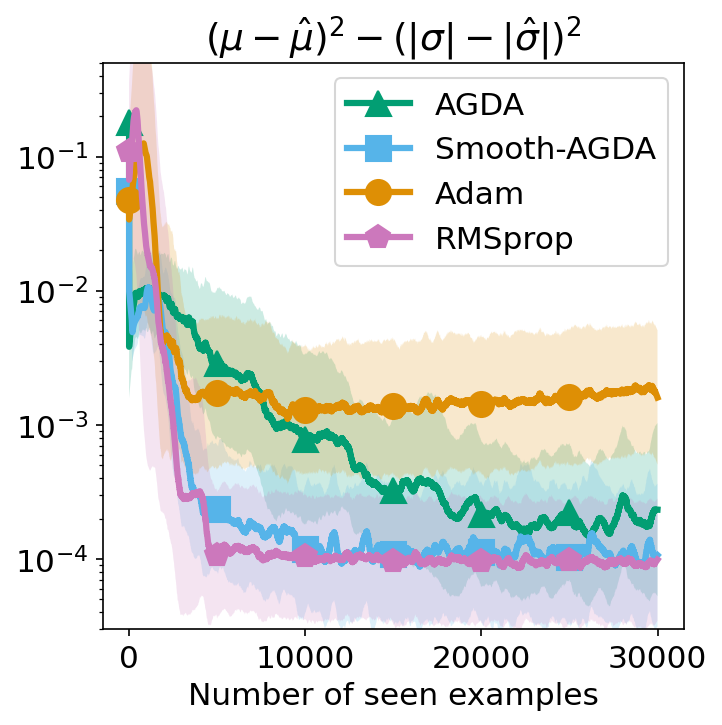}
    \includegraphics[width=0.3\textwidth]{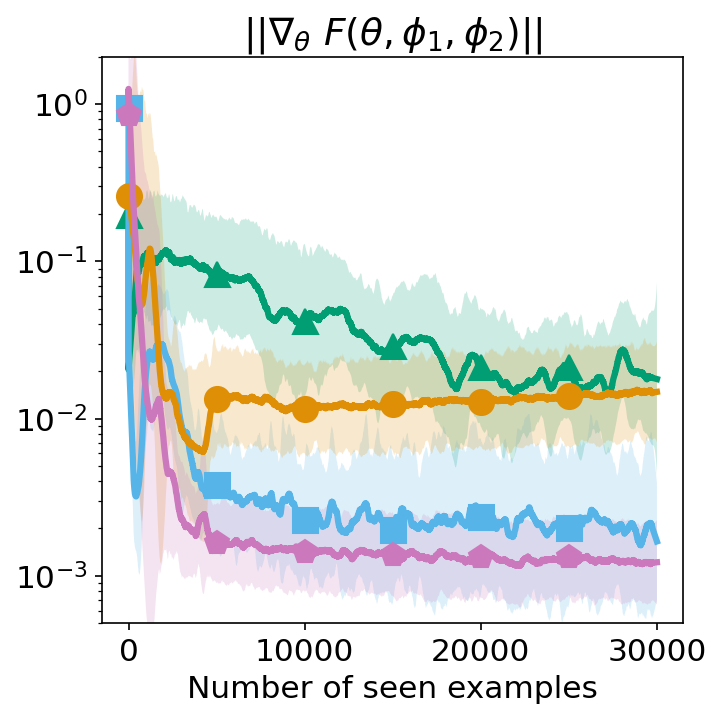}
    \caption{\small ReLU Network generator for a regularized WGAN~(same settings as for Figure~\ref{fig:WGAN_res}). Each algorithm is tuned to yield best performance, with a procedure similar to the one in Figure~\ref{fig:WGAN_res}. The gradient with respect to the discriminator evolves very similarly to the last example, with fast convergence to a non-zero value.}
    \label{fig:WGAN_neural}
\end{figure}

\paragraph{Toy WGAN with neural generator.} Inspired by ~\citep{lei2020sgd}, we consider a regularized WGAN with a neural network as generator. For ease of comparison, we leave all the problem settings identical to last paragraph, and only change the generator $G_{\mu,\sigma}$ to $G_\theta$, where $\theta$ are the parameters of a small neural network~(one hidden layer with five neurons and ReLU activations). After careful tuning for each algorithm, we observe from Figure~\ref{fig:WGAN_neural} that Stoc-Smoothed-AGDA still performs significantly better than vanilla Stoc-AGDA and Adam in this setting. The adaptiveness~(without momentum) of RMSprop is able to yield slightly better results. This is not surprising, as adaptive methods are the de facto optimizers of choice in generative adversarial nets. Hence, a clear direction of future research is to combine adaptiveness and Smoothed-AGDA.

\begin{figure}[t]
    \centering
    \includegraphics[width=0.3\textwidth]{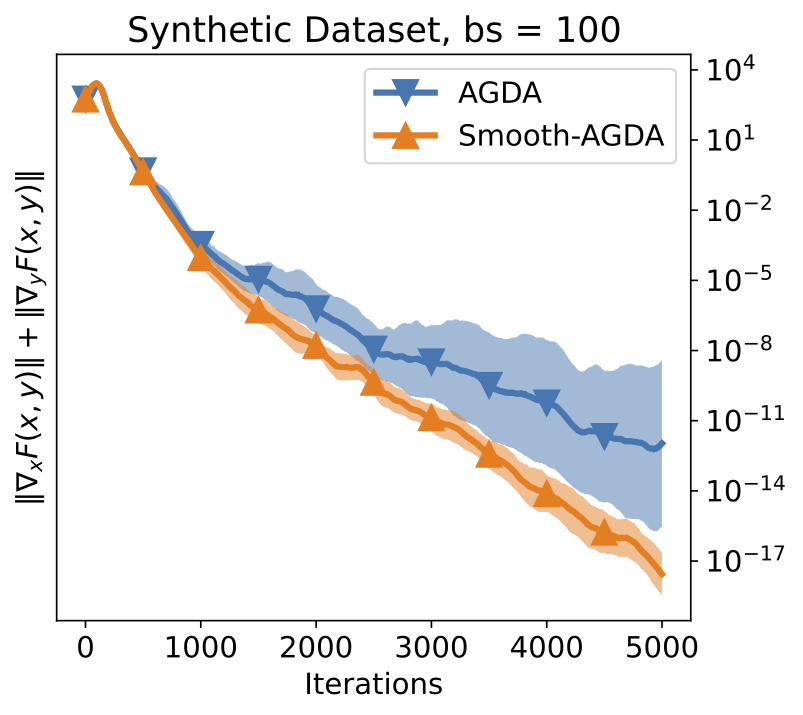}
        \includegraphics[width=0.3\textwidth]{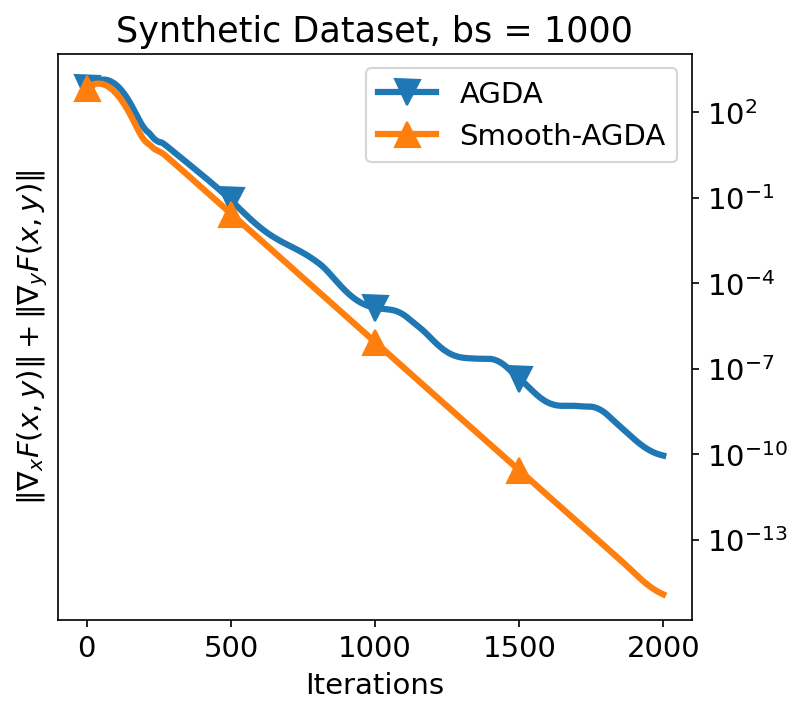}
    \caption{\small Robust non-linear regression on a synthetic Gaussian Dataset. Using $\tau_1 = 5e-4, \tau_2 = 5$ for both AGDA and Smoothed-AGDA, we notice a performance improvement for the latter using $\beta=0.5, p=10$. }
    \label{fig:rob_reg_1}
\end{figure}

\paragraph{Robust non-linear regression.} The experiments above suggest that Smoothed-AGDA accelerates convergence of AGDA. We found that this holds true also outside the WGAN setting: in this last paragraph, we show how this accelerated behavior in a few robust regression problems. We first consider a synthetic dataset of $1000$ datapoints $z$ in $500$ dimensions, sampled from a Gaussian distribution with mean zero and variance 1. The target values $y_0$ are sampled according to a random noisy linear model. We consider fitting this synthetic dataset with a two-hidden-layer ReLU network~(256 units in the first layer, 64 in the second): $\text{net}_x(z)$ with $x$ being the parameter. For the robustness part, we proceed in the standard way~(see e.g.\citep{adolphs2019local}) and add the concave objective $-\frac{\lambda}{2}\|y-y_0\|^2$ to the loss:
$$F(x,y) =\frac{1}{n}\sum_{i=1}^n \frac{1}{2}\|\text{net}_x(z)-y\|^2-\frac{\lambda}{2}\|y-y_0\|^2, $$
where we chose $\lambda = 1$. In this experiement, we compare the performance of AGDA and Smoothed-AGDA under the same stepsize $\tau_1,\tau_2$. From Figure~\ref{fig:rob_reg_1}, we observe that Smoothed-AGDA has much faster convergence than AGDA both in the stochastic and deterministic setting (i.e. with full batch).

\section{Conclusion}
We established faster convergence rates for two \textit{single-loop} algorithms under milder assumption than strong concavity. In particular, we showed that stochastic AGDA can achieve $O(\epsilon^{-4})$ sample complexity without large batch sizes. In addition, we established a better complexity in terms of the dependency to the condition number for Smooth AGDA in both the stochastic and deterministic settings, which also improves over other single-loop algorithms for nonconvex-strongly-concave minimax optimization. There are a few questions worth further investigations, e.g.: (a) what is the lower complexity bound for optimization under PL condition; (b) whether single-loop algorithms can always achieve the rate as fast as multi-loop algorithms; (c) how to design adaptive algorithms for minimax problems without strong concavity.


\bibliographystyle{plainnat}
\bibliography{ref}

\clearpage

\appendix

{\Large \bf \centering Appendix}

\section{Useful Lemmas}

\begin{lemma} [Lemma B.2 \citep{lin2020near}] \label{lin's lemma}
	Assume $f(\cdot, y)$ is $\mu_x$-strongly convex for $\forall y\in\mathbb{R}^{d_2}$ and $f(x, \cdot)$ is $\mu_y$-strongly concave for $\forall x\in\mathbb{R}^{d_1}$ (we will later refer to this as $(\mu_x, \mu_y)$-SC-SC)) and $f$ is $l$-Lipschitz smooth. Then we have
	
	\begin{enumerate}[label=\alph*)]
		\item $y^*(x) = \argmax_{y\in\mathbb{R}^{d_2}} f(x,y)$ is $\frac{l}{\mu_y}$-Lipschitz;
		\item $\Phi(x) = \max_{y\in\mathbb{R}^{d_2}} f(x,y)$ is $\frac{2l^2}{\mu_y}$-Lipschitz smooth and $\mu_x$-strongly convex with $\nabla\Phi(x) = \nabla_xf(x, y^*(x))$;
		\item $x^*(y) = \argmin_{x\in\mathbb{R}^{d_1}} f(x,y)$ is $\frac{l}{\mu_x}$-Lipschitz;
		\item $\Psi(y) = \min_{x\in\mathbb{R}^{d_1}} f(x,y)$ is $\frac{2l^2}{\mu_x}$-Lipschitz smooth and $\mu_y$-strongly concave with $\nabla\Psi(y) = \nabla_yf(x^*(y), y)$.
	\end{enumerate}
	
\end{lemma}

\begin{lemma}[\citet{karimi2016linear}]  \label{PL to EB QG}
	If $f(\cdot)$ is l-smooth and it satisfies PL condition with constant $\mu$, i.e.
	\begin{equation*}
		\Vert \nabla f(x)\Vert^2 \geq 2\mu [f(x) - \min_x f(x)], \forall x,
	\end{equation*}
	then it also satisfies error bound (EB) condition with $\mu$, i.e.
	\begin{equation*}
		\Vert \nabla f(x)\Vert \geq \mu \Vert x_p-x\Vert, \forall x,
	\end{equation*}
	where $x_p$ is the projection of $x$ onto the optimal set, and it satisfies quadratic growth (QG) condition with $\mu$, i.e.
	\begin{equation*}
		f(x)-\min_x f(x)\geq \frac{\mu}{2}\Vert x_p-x\Vert^2, \forall x.
	\end{equation*}
\end{lemma}

\begin{lemma} [\citet{nouiehed2019solving}]  \label{g smooth}
	Under Assumption \ref{Lipscthitz gradient} and \ref{PL assumption}, define $\Phi(x) = \max_y f(x, y)$ then 
	\begin{enumerate}[label=\alph*)]
		\item for any $x_1$, $x_2$, and $y^*(x_1) \in \Argmax_y f(x_1, y)$, there exists some $y^*(x_2)\in \Argmax_y f(x_2, y)$ such that 
		\begin{equation*}
			\left\|y_1^* - y_2^*\right\| \leq \frac{l}{2 \mu}\left\|x_1 - x_2\right\|.
		\end{equation*}
		\item $\Phi(\cdot)$ is $L$-smooth with $L:=l+\frac{l\kappa}{2}$ with $\kappa = \frac{l}{\mu}$ and $\nabla \Phi(x) = \nabla_x f(x, y^*(x))$ for any $y^*(x)\in \Argmax_y f(x,y)$.
	\end{enumerate}
\end{lemma}

Now we present a Theorem adopted from \citep{yang2020global}. Under the two-sided PL condition, it captures the convergence of AGDA with dual updated first\footnote{The update is equivalent to applying AGDA with primal variable update first to $\min_y \max_x -f(x,y)$, so its convergence is a direct result from \citep{yang2020global}. We believe similar convergence rate to Theorem \ref{thm two-sided pl} holds for AGDA with $x$ update first. But for simplicity, here we consider update (\ref{reverse agda}) without additional derivation.}:
\begin{align} \nonumber
	& y^{k+1} = y^k + \tau_2\nabla_y f(x^k, y^k), \\ \label{reverse agda}
	& x^{k+1} = x^k - \tau_1\nabla_x f(x^k, y^{k+1}).
\end{align}

\begin{theorem} [\citet{yang2020global}] \label{thm two-sided pl}
	
	Consider a minimax optimization problem under Assumption \ref{stochastic gradients}: $$ \min _{x \in \mathbb{R}^{d_{1}}} \max _{y \in \mathbb{R}^{d_{2}}} f(x, y) \triangleq \mathbb{E}[F(x, y ; \xi)].$$ Suppose the function $f$ is $l$-smooth, $f(\cdot,y)$ satisfies the PL condition with constant $\mu_1$ and  $-f(x,\cdot)$  satisfies the PL condition with constant $\mu_2$ for any $x$ and $y$. Define
	$$P_k = \mathbb{E}[\Psi^* - \Psi(y_t)] + \frac{1}{10} \mathbb{E}[f(x^k, y^k) -\Psi(x^k) ]$$
	with $\Psi(y) = \min_x f(x,y)$ and $\Psi^* = \max_y \Psi(y)$. If we run Stoc-AGDA (with update rule (\ref{reverse agda})) with stepsizes $\tau_{1} \leq \frac{1}{l}$ and $ \tau_{2}\leq \frac{\mu_{1}^{2} \tau_{1}}{18 l^{2}}$, then 
	\begin{equation}
		P_{k} \leq\left(1-\frac{\mu_2\tau_2}{2}\right)^{k} P_{0} + \frac{23l^2\tau_2^2/\mu_1 + l\tau_1^2}{10\mu_2\tau_2}\sigma^2.
	\end{equation}
	In the deterministic setting, e.g. $\sigma = 0$, if we run AGDA with stepsizes $\tau_{1} = \frac{1}{l}$ and $ \tau_{2} = \frac{\mu_{1}^{2}}{18 l^{3}}$ then 
	\begin{equation}
		P_{k} \leq\left(1-\frac{\mu_{1}^2 \mu_{2}}{36 l^{3}}\right)^{k} P_{0}.
	\end{equation}
\end{theorem}

\begin{definition} [Moreau Envelope]
	The Moreau envelope of a function $\Phi$ with a parameter $\lambda>0$ is: 
	$$
	\Phi_{\lambda}(x) = \min_{z\in\mathbb{R}^{d_1}}  \Phi(z) + \frac{1}{2\lambda}\Vert z - x \Vert^2.$$
\end{definition}

The proximal point of $x$ is defined as: $\operatorname{prox}_{\lambda \Phi}(x)=\argmin_{z\in\mathbb{R}^{d_1}}\left\{\Phi(z)+\frac{1}{2 \lambda}\|z-x\|^{2}\right\}$. The gradients of $\Phi$ and and $\Phi_\lambda$ are closely related by the following well-known lemma; see e.g. \citep{drusvyatskiy2019efficiency}.  

\begin{lemma} \label{lemma moreau}
	When $F$ is differentiable and $\ell$-Lipschitz smooth, for $\lambda \in(0,1 / \ell)$ we have
	$		\nabla \Phi(\prox_{\lambda F} (x)) = \nabla  \Phi_\lambda(x) = \lambda^{-1}(x - \prox_{\lambda \Phi}(x)).
	$
\end{lemma}

\vspace{10mm}

\textbf{Proof of Proposition \ref{prop conversion}}

\begin{proof} We will prove Part (a) and (b) separately. \\
	\textbf{Part (a)}:
	If we can find $\hat{y}$ such that $\max_yf(\hat{x}, y) - f(\hat{x}, \hat{y}) \leq \frac{\epsilon^2}{l\kappa}$, then as $\|\nabla_y f(\hat{x}, y^*(\hat{x}))\| = 0$,
	\begin{align*}
		\|\nabla_y f(\hat{x}, \hat{y})\| \leq  \|\nabla_y f(\hat{x}, \hat{y}) - \nabla_y f(\hat{x}, y^*(\hat{x}))\| 
		\leq  l\|\hat{y} - y^*(\hat{x})\| 
		\leq  l\sqrt{\frac{2}{\mu}[\max_yf(\hat{x}, y) - f(\hat{x}, \hat{y})]} \leq \sqrt{2}\epsilon,
	\end{align*}
	where in the first inequality we fix $y^*(x)$ to the projection from $\hat{y}$ to $\Argmax_y f(\hat{x}, y)$, in the second inequality we use Lipschitz smoothness, and in the third inequality we use PL condition and Lemma \ref{PL to EB QG}. Also, 
	\begin{align*}
		\|\nabla_x f(\hat{x}, \hat{y})\| \leq & \|\nabla_x f(\hat{x}, y^*(\hat{x})\| + \|\nabla_x f(\hat{x}, \hat{y}) - \nabla_x f(\hat{x}, y^*(\hat{x}))\| \\
		\leq & \|\nabla \Phi(\hat{x})\| + l\|\hat{y} - y^*(\hat{x})\| \\
		\leq & \|\nabla \Phi(\hat{x})\| + l\sqrt{\frac{2}{\mu}[\max_yf(\hat{x}, y) - f(\hat{x}, \hat{y})]} \leq (1+\sqrt{2})\epsilon,
	\end{align*}
	where in the second inequality we use Lemma \ref{g smooth}. Therefore, our goal is to find $\hat{y}$ such that $\max_yf(\hat{x}, y) - f(\hat{x}, \hat{y}) \leq \frac{\epsilon^2}{l\kappa}$  by applying (stochastic) gradient ascent to $f(\hat{x}, \cdot)$ from initial point $\Tilde{y}$.
	
	\textbf{Deterministic case}: Since $\|\nabla_y f(\hat{x}, \Tilde{y})\| \leq \epsilon'$, we have $\max_yf(\hat{x}, y) - f(\hat{x}, \Tilde{y}) \leq \frac{\epsilon'^2}{2\mu}$ by PL condition. Let $y^k$ denote $k$-th iterates of gradient ascent from initial point $\Tilde{y}$ with stepsize $\frac{1}{l}$. Then by \citep{karimi2016linear}
	\begin{equation*}
		\max_yf(\hat{x}, y) - f(\hat{x}, y^k) \leq \left(1-\frac{1}{\kappa}\right)^k\left[\max_yf(\hat{x}, y) - f(\hat{x}, \Tilde{y})\right].
	\end{equation*}
	So after $O\left(\kappa \log\left(\frac{\kappa\epsilon'}{\epsilon} \right)\right)$, we can find the point we want. 
	
	\textbf{Stochastic Case}: Let $y^k$ denote $k$-th iterates of stochastic gradient ascent from initial point $\Tilde{y}$ with stepsize $\tau \leq \frac{1}{l}$. Then by Lemma A.4 in \citep{yang2020catalyst}
	\begin{equation*}
		\mathbb{E}\left[\max_yf(\hat{x}, y)) - f(\hat{x}, y^{k+1})\right] \leq (1-\mu\tau)\mathbb{E}\left[\max_yf(\hat{x}, y)) - f(\hat{x}, y^k)\right] + \frac{l \tau^{2}}{2} \sigma^{2}, 
	\end{equation*}
	which implies 
	\begin{equation*}
		\mathbb{E}\left[\max_yf(\hat{x}, y)) - f(\hat{x}, y^{k})\right] \leq (1-\mu\tau)^k\mathbb{E}\left[\max_yf(\hat{x}, y)) - f(\hat{x}, \Tilde{y})\right] + \frac{\kappa \tau}{2} \sigma^{2}. 
	\end{equation*}
	So with $\tau = \min\left\{\frac{1}{l}, \Theta\left(\frac{\epsilon^2}{l\kappa^2\sigma^2} \right) \right\}$, we can find the point we want with a complexity of $O\left(\kappa \log\left(\frac{\kappa\epsilon'}{\epsilon} \right) + \kappa^3\sigma^2\log\left(\frac{\kappa\epsilon'}{\epsilon} \right)\epsilon^{-2} \right)$.
	\vspace{5mm}\\ 
	\textbf{Part (b)}:
	We first look at $\Phi_{1/2l}(\Tilde{x}) = \min_z \Phi(z) + l\|z-\Tilde{x}\|^2$. Then by Lemma 4.3 in \citep{drusvyatskiy2019efficiency}, 
	\begin{align} \nonumber
		& \Vert \nabla \Phi_{1/2l}(\Tilde{x})\Vert^2 \\ \nonumber
		= & 4l^2 \Vert \Tilde{x} - \prox_{\Phi/2l}(\Tilde{x})\Vert^2  \\ \nonumber
		\leq & 8l [\Phi(\Tilde{x}) - \Phi(\prox_{\Phi/2l}(\Tilde{x})) -l\|\prox_{\Phi/2l}(\Tilde{x}) - \Tilde{x}\|^2] \\ \nonumber
		= & 8l \left[\Phi(\Tilde{x}) - f(\Tilde{x},\Tilde{y}) + f(\Tilde{x},\Tilde{y}) -  f(\prox_{\Phi/2l}(\Tilde{x}),\Tilde{y}) + f(\prox_{\Phi/2l}(\Tilde{x}),\Tilde{y}) -   \Phi(\prox_{\Phi/2l}(\Tilde{x})) -l\|\prox_{\Phi/2l}(\Tilde{x}) - \Tilde{x}\|^2 \right] \\ \label{moreau to gradient}
		\leq & 8l \left[\frac{1}{2\mu}\|\nabla_y f(\Tilde{x},\Tilde{y})\|^2 + f(\Tilde{x},\Tilde{y}) - f(\prox_{\Phi/2l}(\Tilde{x}),\Tilde{y}) - l\|\prox_{\Phi/2l}(\Tilde{x}) - \Tilde{x}\|^2  \right]
	\end{align}
	where in the first inequality we use the $l$-strong-convexity in $x$ of $\Phi(x) + l\|x -\Tilde{x}\|^2$, in the second inequality we use $\Phi(\Tilde{x}) - f(\Tilde{x},\Tilde{y}) \leq \frac{1}{2\mu}\|\nabla_y f(\Tilde{x},\Tilde{y})\|^2$ by PL condition, and $ f(\prox_{\Phi/2l}(\Tilde{x}),\Tilde{y}) -   \Phi(\prox_{\Phi/2l}(\Tilde{x})) \leq 0$. Note that by defining $\hat{f}(x, y) = f(x, y) + l\|x-\Tilde{x}\|^2$, we have 
	\begin{align*}
		f(\Tilde{x},\Tilde{y}) - f(\prox_{\Phi/2l}(\Tilde{x}),\Tilde{y}) - l\|\prox_{\Phi/2l}(\Tilde{x}) - \Tilde{x}\|^2 = & \hat{f}(\Tilde{x}, \Tilde{y}) - \hat{f}(\prox_{\Phi/2l}(\Tilde{x}),\Tilde{y}) \\
		\leq & \langle \nabla_x f(\Tilde{x}, \Tilde{y}), x-\prox_{\Phi/2l}(\Tilde{x})\rangle -\frac{l}{2}\|x- \prox_{\Phi/2l}(\Tilde{x})\|^2 \\
		\leq &  \frac{1}{2l}\|\nabla_x \hat{f}(\Tilde{x}, \Tilde{y})\|^2 + \frac{l}{2}\|x- \prox_{\Phi/2l}(\Tilde{x})\|^2 -\frac{l}{2}\|x- \prox_{\Phi/2l}(\Tilde{x})\|^2 \\
		\leq & \frac{1}{2l}\|\nabla_x \hat{f}(\Tilde{x}, \Tilde{y})\|^2 = \frac{1}{2l}\|\nabla_x f(\Tilde{x}, \Tilde{y})\|^2,
	\end{align*}
	where in the second inequality we use $l$-strong-convexity in $x$ of $\hat{f}(x, y)$. Plugging into (\ref{moreau to gradient}),
	\begin{equation} \label{initial moreau bdn}
		\Vert \nabla \Phi_{1/2l}(\Tilde{x})\Vert^2 = 4l^2 \Vert \Tilde{x} - \prox_{\Phi/2l}(\Tilde{x})\Vert^2  \leq 4\kappa\|\nabla_y f(\Tilde{x},\Tilde{y})\|^2 + 4\|\nabla_x f(\Tilde{x},\Tilde{y})\|^2 \leq 8\epsilon.
	\end{equation}
	If we can find $\hat{x}$ such that $\|\prox_{\Phi/2l}(\Tilde{x}) - \hat{x}\|\leq \frac{\epsilon}{\kappa l}$, then
	\begin{equation*}
		\|\nabla \Phi(\hat{x})\| \leq \|\nabla \Phi(\prox_{\Phi/2l}(\Tilde{x}))\|  + \|\nabla \Phi(\hat{x}) - \nabla \Phi(\prox_{\Phi/2l}(\Tilde{x}))\| \leq  \Vert \nabla \Phi_{1/2l}(\Tilde{x})\Vert + 2\kappa l \|\prox_{\Phi/2l}(\Tilde{x}) - \hat{x}\| \leq (2\sqrt{2}+2)\epsilon,
	\end{equation*}
	where in the second inequality we use Lemma \ref{g smooth} and Lemma \ref{lemma moreau}. Note that $\prox_{\Phi/2l}(\Tilde{x})$ is the solution to $\min_x \Phi(x) + l\|x- \Tilde{x}\|^2$, which is equivalent to 
	\begin{equation} \label{sub for measure change}
		\min_x\max_y f(x, y) + l\|x- \Tilde{x}\|^2.
	\end{equation}
	This minimax problem is $l$-strongly convex about $x$, $\mu$-PL about $y$ and $3l$-smooth. Therefore, we can use (stochastic) alternating gradient descent ascent (AGDA) to find $\hat{x}$ such that $\|\prox_{\Phi/2l}(\Tilde{x}) - \hat{x}\|\leq \frac{\epsilon}{\kappa l}$ from initial point $(\Tilde{x}, \Tilde{y})$.
	
	\textbf{Deterministic case}: Let $(x^k, y^k)$ denote $k$-th iterates of AGDA with $y$ updated first from initial point $(\Tilde{x}, \Tilde{y})$ on function (\ref{sub for measure change}). Define $\hat{\Phi}(x) = \max_y \hat{f}(x, y) = \max_y f(x, y) + l\|x-\Tilde{x}\|^2 $, $\hat{\Psi}(y) = \min_x \hat{f}(x, y) = \min_x f(x,y)+ l\|x-\Tilde{x}\|^2$ and $\hat{\Psi}^* = \max_y \hat{\Psi}(y) $. We also denote $x^* = \argmin_x \hat{\Phi}(x) = \prox_{\Phi/2l}(\Tilde{x})$. Then we define $P_k = \hat{\Psi}^* - \hat{\Psi}(y^k) + \frac{1}{10}\left[ \hat{f}(x^k, y^k) - \hat{\Psi}(y^k)\right]$. Note that 
	\begin{align} \label{convert initial}
		P_0 = &\hat{\Psi}^* - \hat{\Psi}(\Tilde{y}) + \frac{1}{10}\left[ \hat{f}(\Tilde{x}, \Tilde{y}) - \hat{\Psi}(\Tilde{y})\right] \leq \hat{\Psi}^* - \hat{\Psi}(\Tilde{y}) + \frac{1}{20l}\|\nabla_x \hat{f}(\Tilde{x}, \Tilde{y})\|^2 \leq \hat{\Psi}^* - \hat{\Psi}(\Tilde{y}) + \frac{\epsilon^2}{20l}.
	\end{align}
	
	Also we note that 
	\begin{align*}
		\hat{\Psi}^* - \hat{\Psi}(\Tilde{y}) = & \max_y\min_x \hat{f}(x, y) - \min_x\hat{f}(x,\Tilde{y}) \\
		= & \max_y\min_x \hat{f}(x, y) - \max_y \hat{f}(\Tilde{x},y) + \max_y \hat{f}(\Tilde{x},y) - \hat{f}(\Tilde{x}, \Tilde{y}) + \hat{f}(\Tilde{x}, \Tilde{y}) - \min_x\hat{f}(x,\Tilde{y})\\
		\leq & \frac{1}{2\mu}\|\nabla_y \hat{f}(\Tilde{x}, \Tilde{y})\|^2 + \frac{1}{2l}\|\nabla_x \hat{f}(\Tilde{x}, \Tilde{y})\|^2 = \frac{1}{2\mu}\|\nabla_y f(\Tilde{x}, \Tilde{y})\|^2 + \frac{1}{2l}\|\nabla_x f(\Tilde{x}, \Tilde{y})\|^2 \leq \frac{1}{l}\epsilon^2,
	\end{align*}
	where in the first inequality we use $\max_y\min_x \hat{f}(x, y) \leq  \max_y \hat{f}(\Tilde{x},y) $, $l$-strong-convexity of $\hat{f}(\cdot, \Tilde{y})$ and $\mu$-PL of $\hat{f}(\Tilde{x}, \cdot)$. Combined with (\ref{convert initial}) we have 
	\begin{equation*}
		P_0 \leq \frac{2\epsilon^2}{l}.
	\end{equation*}
	Then we note that 
	\begin{align*}
		\|x^k - x^*\|^2 \leq & 2\|x^k - x^*(y^k)\|^2 + 2\|x^*(y^k)-x^*\|^2 \leq  \frac{4}{l}[\hat{f}(x^k, y^k) - \hat{\Psi}(y^k)] + 18\|y^k-y^*\|^2 \\
		\leq & \frac{4}{l}[\hat{f}(x^k, y^k) - \hat{\Psi}(y^k)] + \frac{18}{\mu}[\hat{\Psi}(y^k) - \hat{\Psi}^*] \leq \frac{40}{\mu} P_k,
	\end{align*}
	where in the second inequality we use $l$-strong-convexity of $\hat{f}(\cdot, y^k)$ and Lemma \ref{lin's lemma}, in the third inequality we use $\mu$-PL of $\hat{\Psi}(\cdot)$ (see e.g. \citep{yang2020global}). Because $\hat{f}(x, y)$ is $l$-strongly convex about $x$, $\mu$-PL about $y$ and $3l$-smooth, it satifies the two-sided PL condition in \citep{yang2020global} and it can be solved by AGDA. By Theorem \ref{thm two-sided pl}, if we choose $\tau_1 = \frac{1}{3l}$ and $\tau_2 = \frac{l^2}{18(3l)^3} = \frac{1}{486l}$, we have
	\begin{equation*}
		P_k \leq (1-\frac{1}{972\kappa})^kP_0,
	\end{equation*}
	Therefore, 
	\begin{equation*}
		\|x^k - x^*\|^2 \leq  \frac{40}{\mu} P_k \leq \frac{40}{\mu}\left(1-\frac{1}{972\kappa}\right)^kP_0 \leq \frac{80\epsilon^2}{\mu l}\left(1-\frac{1}{972\kappa}\right)^k.
	\end{equation*}
	So after $O(\kappa\log\kappa)$ iterations we have $\|x^k - x^*\|^2 \leq  \frac{\epsilon^2}{\kappa^2 l^2}$.
	
	\textbf{Stochastic case}: By Theorem \ref{thm two-sided pl}, if we choose $\tau_1 \leq \frac{1}{3l}$ and $\tau_2 = \frac{l^2\tau_1}{18(3l)^2} = \frac{\tau_1}{162}$, we have
	\begin{equation*}
		P_k \leq \left(1-\frac{\mu\tau_2}{2}\right)^kP_0 + O(\kappa\tau_2\sigma^2).
	\end{equation*}
	With $\tau_2 = \min\left\{\frac{1}{486l}, \Theta\left(\frac{\epsilon^2}{\kappa^4l\sigma^2} \right)\right\}$ and $\tau_1 = 162\tau_2$, we have $\|x^k - x^*\|^2 \leq  \frac{\epsilon^2}{\kappa^2 l^2}$ after $O\left(\kappa \log(\kappa) + \kappa^5\sigma^2\log(\kappa)\epsilon^{-2} \right)$ iterations.

\end{proof}

\section{Proofs for Stochastic AGDA}
\label{apdx agda}

\textbf{Proof of Theorem \ref{thm agda}}

\begin{proof}

	Because $\Phi$ is $L$-smooth with $L = l+\frac{l\kappa}{2}$ by Lemma \ref{g smooth}, we have the following by Lemma A.4 in \citep{yang2020global} 
	\begin{align*}
		\Phi(x_{t+1}) \leq& \Phi(x_t)+\langle \nabla \Phi(x_t), x_{t+1}-x_t \rangle +\frac{L}{2}\Vert x_{t+1}-x_t\Vert^2\\
		=& \Phi(x_t) - \tau_1\langle \nabla \Phi(x_t), G_x(x_t, y_t,\xi_{t1}) \rangle +\frac{L}{2}\tau_1^2\Vert G_x(x_t, y_t,\xi_{1}^t) \Vert^2.
	\end{align*}
	Taking expectation of both side and use Assumption \ref{stochastic gradients}, we get
	\begin{align}  \nonumber
		\mathbb{E}[ \Phi(x_{t+1})] \leq &\mathbb{E}[\Phi(x_t)] - \tau_1\mathbb{E}[\langle \nabla \Phi(x_t), \nabla_xf(x_t,y_t) \rangle] + \frac{L}{2}\tau_1^2\mathbb{E}[\Vert G_x(x_t, y_t,\xi_{1}^t) \Vert^2] \\ \nonumber
		\leq& \mathbb{E}[\Phi(x_t)] - \tau_1\mathbb{E}[\langle \nabla \Phi(x_t), \nabla_xf(x_t,y_t) \rangle] +\frac{L}{2}\tau_1^2\mathbb{E}[\Vert \nabla_xf(x_t, y_t) \Vert^2] + \frac{L}{2}\tau_1^2\sigma^2 \\ \nonumber
		\leq& \mathbb{E}[\Phi(x_t)] - \tau_1\mathbb{E}[\langle \nabla \Phi(x_t), \nabla_xf(x_t,y_t) \rangle] +\frac{\tau_1}{2}\mathbb{E}[\Vert \nabla_xf(x_t, y_t) \Vert^2] + \frac{L}{2}\tau_1^2\sigma^2\\  \label{new stoc primal decrease}
		\leq& \mathbb{E}[\Phi(x_t)] - \frac{\tau_1}{2}\mathbb{E}\Vert \nabla \Phi(x_t) \Vert^2 + \frac{\tau_1}{2}\mathbb{E}\Vert \nabla_x f(x_t, y_t) - \nabla \Phi(x_t)\Vert^2 +\frac{L}{2}\tau_1^2\sigma^2,
	\end{align}
	where in the second inequality we use Assumption \ref{stochastic gradients}, and in the third inequality we use $\tau_1\leq 1/L$. By smoothness of $f(x, \cdot)$, we have
	\begin{align} \nonumber
		f(x_{t+1}, y_{t+1}) \geq & f(x_{t+1}, y_t) + \langle \nabla_yf(x_{t+1},y_t), y_{t+1}-y_t\rangle - \frac{l}{2}\|y_{t+1}-y_t\|^2 \\ \nonumber
		\geq & f(x_{t+1}, y_t) +  \tau_2\langle \nabla_yf(x_{t+1},y_t), G_y(x_{t+1}, y_t, \xi_2^t)\rangle - \frac{l\tau_2^2}{2}\|G_y(x_{t+1}, y_t, \xi_2^t)\|^2 .
	\end{align}
	Taking expectation, as $\tau_2\leq \frac{1}{l}$
	\begin{align}\nonumber
		\mathbb{E} f(x_{t+1},y_{t+1}) - \mathbb{E}f(x_{t+1},y_t) \geq & \tau_2\mathbb{E}\|\nabla_y f(x_{t+1}, y_t)\|^2 - \frac{l\tau_2^2}{2}\mathbb{E}\|\nabla_y f(x_{t+1}, y_t)\|^2 - \frac{l\tau_2^2}{2}\sigma^2 \\ \label{new stoc y update}
		\geq & \frac{\tau_2}{2}\mathbb{E}\|\nabla_y f(x_{t+1}, y_t)\|^2 - \frac{l\tau_2^2}{2}\sigma^2 .
	\end{align}
	By smoothness of $f(\cdot, y)$, we have
	\begin{align} \nonumber
		f(x_{t+1},y_t) \geq & f(x_t, y_t) + \langle \nabla_x f(x_t, y_t), x_{t+1}-x_t\rangle - \frac{l}{2}\|x_{t+1}-x_t\|^2 \\ \nonumber
		\geq & f(x_t, y_t) - \tau_1\langle \nabla_x f(x_t, y_t), G_xf(x_t, y_t, \xi_1^t)\rangle - \frac{l\tau_1^2}{2}\|G_x(x_t, y_t, \xi_1^t)\|^2.
	\end{align}
	Taking expectation, as $\tau_1 \leq \frac{1}{l}$
	\begin{align} \nonumber
		\mathbb{E} f(x_{t+1},y_{t}) - \mathbb{E}f(x_{t},y_t) \geq & -\tau_1\mathbb{E}\|\nabla_x f(x_t, y_t)\| - \frac{l\tau_1^2}{2}\mathbb{E}\|\nabla_x f(x_t, y_t)\| -\frac{l\tau_1^2}{2}\sigma^2
		\\  \label{new stoc x update}
		\geq &-\frac{3\tau_1}{2}\mathbb{E}\|\nabla_x f(x_t, y_t)\|^2 -\frac{l\tau_1^2}{2}\sigma^2 .
	\end{align}
	Therefore, summing (\ref{new stoc x update}) and (\ref{new stoc y update}) together
	\begin{equation}
		\mathbb{E} f(x_{t+1},y_{t+1}) - \mathbb{E}f(x_{t},y_t) \geq \frac{\tau_2}{2}\mathbb{E}\|\nabla_y f(x_{t+1}, y_t)\|^2 -\frac{3\tau_1}{2}\mathbb{E}\|\nabla_x f(x_t, y_t)\|^2 - \frac{l\tau_1^2}{2}\sigma^2 - \frac{l\tau_2^2}{2}\sigma^2.
	\end{equation}
	Now we consider the following potential function, for some $\alpha>0$ which we will pick later
	\begin{equation} \label{new stoc obj decrease}
		V_t = V(x_t, y_t) = \Phi(x_t) + \alpha [\Phi(x_t) - f(x_t, y_t)] = (1+\alpha) \Phi(x_t) - \alpha f(x_t, y_t). 
	\end{equation}
	Then by combining (\ref{new stoc obj decrease}) and (\ref{new stoc primal decrease}) we have 
	\begin{align}  \nonumber
		\mathbb{E} V_t - \mathbb{E} V_{t+1} 
		\geq& \frac{\tau_{1}}{2}(1+\alpha) \mathbb{E}\left\|\nabla \Phi\left(x_{t}\right)\right\|^{2} - \frac{\tau_{1}}{2}(1+\alpha) \mathbb{E}\left\|\nabla_{x} f\left(x_{t}, y_{t}\right)-\nabla \Phi\left(x_{t}\right)\right\|^{2}  + \frac{\tau_2\alpha}{2}\mathbb{E}\|\nabla_y f(x_{t+1}, y_t)\|^2 - \\ \nonumber
		&\frac{3\tau_1\alpha}{2}\mathbb{E}\|\nabla_x f(x_t, y_t)\|^2 - 
		\left[ \frac{L(1+\alpha)}{2} \tau_{1}^{2} + \frac{l\tau_2^2\alpha}{2} + \frac{l\tau_1^2\alpha}{2}\right] \sigma^2 \\
		\geq &  \left[\frac{\tau_{1}}{2}(1+\alpha) -3\tau_1\alpha \right]\mathbb{E}\left\|\nabla \Phi\left(x_{t}\right)\right\|^{2}  - \left[  \frac{\tau_{1}}{2}(1+\alpha) + 3\tau_1\alpha\right]\mathbb{E}\left\|\nabla_{x} f\left(x_{t}, y_{t}\right)-\nabla \Phi\left(x_{t}\right)\right\|^{2} + \\ \nonumber
		& \frac{\tau_2\alpha}{4}\mathbb{E}\|\nabla_y f(x_{t}, y_t)\|^2 - \frac{\tau_2\alpha}{2}\mathbb{E}\|\nabla_y f(x_{t+1}, y_t) - \nabla_y f(x_{t}, y_t) \|^2 - 
		\left[ \frac{L(1+\alpha)}{2} \tau_{1}^{2} + \frac{l\tau_2^2\alpha}{2} + \frac{l\tau_1^2\alpha}{2}\right] \sigma^2 \\ \nonumber
		\geq &  \left[\frac{\tau_{1}}{2}(1+\alpha) -3\tau_1\alpha \right]\mathbb{E}\left\|\nabla \Phi\left(x_{t}\right)\right\|^{2}  - \left[  \frac{\tau_{1}}{2}(1+\alpha) + 3\tau_1\alpha\right]\mathbb{E}\left\|\nabla_{x} f\left(x_{t}, y_{t}\right)-\nabla \Phi\left(x_{t}\right)\right\|^{2} + \\ \nonumber
		& \frac{\tau_2\alpha}{4}\mathbb{E}\|\nabla_y f(x_{t}, y_t)\|^2 - \frac{\tau_2\alpha}{2}l^2\mathbb{E}\|x_{t+1}-x_t \|^2 - 
		\left[ \frac{L(1+\alpha)}{2} \tau_{1}^{2} + \frac{l\tau_2^2\alpha}{2} + \frac{l\tau_1^2\alpha}{2}\right] \sigma^2 \\ \nonumber
		\geq &  \left[\frac{\tau_{1}}{2}(1+\alpha) -3\tau_1\alpha \right]\mathbb{E}\left\|\nabla \Phi\left(x_{t}\right)\right\|^{2}  - \left[  \frac{\tau_{1}}{2}(1+\alpha) + 3\tau_1\alpha\right]\mathbb{E}\left\|\nabla_{x} f\left(x_{t}, y_{t}\right)-\nabla \Phi\left(x_{t}\right)\right\|^{2} + \\ \nonumber
		& \frac{\tau_2\alpha}{4}\mathbb{E}\|\nabla_y f(x_{t}, y_t)\|^2 - \frac{\tau_2\alpha}{2}l^2\tau_1^2\mathbb{E}\|\nabla_xf(x_t, y_t)\|^2 - 
		\left[ \frac{L(1+\alpha)}{2} \tau_{1}^{2} + \frac{l\tau_2^2\alpha}{2} + \frac{l\tau_1^2\alpha}{2} + \frac{\tau_2}{2}\alpha l^2\tau_1^2\right] \sigma^2 \\ \nonumber
		\geq &  \left[\frac{\tau_{1}}{2}(1+\alpha) -3\tau_1\alpha - \tau_2\alpha l^2\tau_1^2 \right]\mathbb{E}\left\|\nabla \Phi\left(x_{t}\right)\right\|^{2}  - \left[  \frac{\tau_{1}}{2}(1+\alpha) + 3\tau_1\alpha + \tau_2\alpha l^2\tau_1^2\right]\mathbb{E}\left\|\nabla_{x} f\left(x_{t}, y_{t}\right)-\nabla \Phi\left(x_{t}\right)\right\|^{2} + \\ \label{new potential bdn}
		& \frac{\tau_2\alpha}{4}\mathbb{E}\|\nabla_y f(x_{t}, y_t)\|^2  - 
		\left[ \frac{L(1+\alpha)}{2} \tau_{1}^{2} + \frac{l\tau_2^2\alpha}{2} + \frac{l\tau_1^2\alpha}{2} + \frac{\tau_2}{2}\alpha l^2\tau_1^2\right] \sigma^2,
	\end{align}
	where in the first inequality we use $\|a+b\|^2\leq 2\|a\|^2+2\|b\|^2$ and $\|a\|^2 \geq \|b\|^2/2 - \|a-b\|^2$, in the second inequality we use smoothness, and in the last inequality we use $\|a+b\|^2\leq \|a\|^2+\|b\|^2$. Note that by smoothness and PL condition, fixing $y^*(x_t)$ to be the projection of $y_t$ to the set $\underset{y}{\operatorname{Argmin}} f(x_t, y)$,
	\begin{equation*}
		\left\|\nabla_{x} f\left(x_{t}, y_{t}\right)-\nabla \Phi\left(x_{t}\right)\right\|^{2} \leq l^2\|y_t-y^*(x_t)\|^2 \leq \kappa^2 \|\nabla_y f(x_t, y_t)\|^2.
	\end{equation*}
	Plugging it into (\ref{new potential bdn}), we get
	\begin{align}\nonumber
		\mathbb{E} V_t - \mathbb{E} V_{t+1} 
		\geq&  \left[\frac{\tau_{1}}{2}(1+\alpha) -3\tau_1\alpha - \tau_2\alpha l^2\tau_1^2 \right]\mathbb{E}\left\|\nabla \Phi\left(x_{t}\right)\right\|^{2}  + \\ \nonumber
		& \left[\frac{\tau_2\alpha}{4} -  \frac{\tau_{1}}{2}(1+\alpha)\kappa^2 - 3\tau_1\alpha\kappa^2 - \tau_2\alpha l^2\tau_1^2\kappa^2\right]\mathbb{E}\|\nabla_y f(x_{t}, y_t)\|^2 - \\
		&\left[ \frac{L(1+\alpha)}{2} \tau_{1}^{2} + \frac{l\tau_2^2\alpha}{2} + \frac{l\tau_1^2\alpha}{2} + \frac{\tau_2}{2}\alpha l^2\tau_1^2\right] \sigma^2.
	\end{align}
	Then we note that when $\alpha = \frac{1}{8}$, $\tau_1 \leq \frac{1}{l}$ and $\tau_2\leq \frac{1}{l}$, 
	\begin{equation*}
		\frac{\tau_{1}}{2}(1+\alpha) -3\tau_1\alpha - \tau_2\alpha l^2\tau_1^2 \geq \frac{\tau_1}{16}.
	\end{equation*}
	Furthermore, when $\tau_1 \leq \frac{\tau_2}{68\kappa^2}$, then 
	\begin{equation*}
		\frac{\tau_2\alpha}{4} -  \frac{\tau_{1}}{2}(1+\alpha)\kappa^2 - 3\tau_1\alpha\kappa^2 - \tau_2\alpha l^2\tau_1^2\kappa^2 \geq \frac{1}{64}\tau_2 \geq \frac{17}{16}\kappa^2\tau_1.
	\end{equation*}
	Also, as $\alpha = \frac{1}{8}$, $\tau_2\leq \frac{1}{l}$ and $\tau_1 = \frac{\tau_2}{68\kappa^2}$
	\begin{equation*}
		\frac{L(1+\alpha)}{2} \tau_{1}^{2} + \frac{l\tau_2^2\alpha}{2} + \frac{l\tau_1^2\alpha}{2} + \frac{\tau_2}{2}\alpha l^2\tau_1^2 \leq 292\kappa^4 l\tau_1^2.
	\end{equation*}
	Therefore, 
	\begin{equation} \label{V dif}
		\mathbb{E} V_t - \mathbb{E} V_{t+1} 
		\geq  \frac{\tau_1}{16}\mathbb{E}\left\|\nabla \Phi\left(x_{t}\right)\right\|^{2} +  \frac{17}{16}\kappa^2\tau_1 \mathbb{E}\|\nabla_y f(x_{t}, y_t)\|^2 - 292\kappa^4 l\tau_1^2\sigma^2.
	\end{equation}
	Telescoping and rearraging, with $a_0 \triangleq \Phi(x_0) - f(x_0, y_0)$,
	\begin{align*}
		\frac{1}{T}\sum_{t=0}^{T-1}\mathbb{E}\left\|\nabla \Phi\left(x_{t}\right)\right\|^{2} \leq &\frac{16}{\tau_1T}[V_0 - \min_{x,y}V(x,y)] + 4762\kappa^4 l\tau_1\sigma^2 \\
		\leq & \frac{16}{\tau_1 T}[\Phi(x_0) -  \Phi^*] + \frac{2}{\tau_1 T}a_0 + 4672\kappa^4 l\tau_1\sigma^2,
	\end{align*}
	where in the second inequality we note that since for any $x$ we can find $y$ such that $\Phi(x) = f(x, y)$,
	\begin{equation*}
		V_0-  \min_{x,y}V(x,y) = \Phi(x_0) + \alpha [\Phi(x_0) - f(x_0, y_0)] - \min_{x,y}\{ \Phi(x) + \alpha [\Phi(x) - f(x, y)]\} = \Phi(x_0)- \Phi^* + \alpha[\Phi(x_0) - f(x_0, y_0)].
	\end{equation*}
	Picking $\tau_1 = \min\bigg\{\frac{\sqrt{\Phi(x_0)- \Phi^*}}{4\sigma\kappa^2\sqrt{Tl}}, \frac{1}{68l\kappa^2} \bigg\}$,
	\begin{align*}
		\frac{1}{T}\sum_{t=0}^{T-1}\mathbb{E}\left\|\nabla \Phi\left(x_{t}\right)\right\|^{2} \leq& \max\bigg\{\frac{4\sigma\kappa^2\sqrt{Tl}}{\sqrt{\Phi(x_0)- \Phi^*}}, 68l\kappa^2 \bigg\}\frac{16}{T}[\Phi(x_0) -  \Phi^*] + \max\bigg\{\frac{4\sigma\kappa^2\sqrt{Tl}}{\sqrt{\Phi(x_0)- \Phi^*}}, 68l\kappa^2 \bigg\}\frac{2}{T}a_0 + \\
		& \frac{\sqrt{\Phi(x_0)- \Phi^*}}{4\sigma\kappa^2\sqrt{Tl}} 4672\kappa^4 l\sigma^2 \\
		\leq & \frac{1088 l\kappa^2}{T}[\Phi(x_0)- \Phi^*] + \frac{136l\kappa^2}{T}a_0 + \frac{8\kappa^2\sqrt{l} a_0}{\sqrt{[\Phi(x_0)- \Phi^*]T}}\sigma + \frac{1232\kappa^2\sqrt{l[\Phi(x_0)- \Phi^*]}}{\sqrt{T}}\sigma.
	\end{align*}
	Here we can pick $\tau_2 = \min\bigg\{\frac{17\sqrt{\Phi(x_0)- \Phi^*}}{\sigma\sqrt{Tl}}, \frac{1}{l} \bigg\}$.

\end{proof}

\textbf{Proof of Corollary \ref{coro agda}}

\begin{proof}
	Similar to the proof of part (a) in Proposition \ref{prop conversion}, fixing $y^*(x_t)$ to be the projection of $x_t$ to $\Argmax_y f(x_t, y)$, we have 
	\begin{align*}
		\|\nabla_x f(x_t, y_t)\|^2 \leq & 2\|\nabla_x f(x_t, y^*(x_t))\|^2 + 2\|\nabla_x f(x_t, y_t) - \nabla_x f(x_t, y^*(x_t))\|^2 \\
		\leq & 2\|\nabla \Phi(x_t)\|^2 + 2l^2\|y_t - y^*(x_t)\|^2 \\
		\leq & 2\|\nabla \Phi(x_t)\| + 2\kappa^2\|\nabla_y f(x_t, y_t)\|^2,
	\end{align*}
	where in the first inequality we use Lemma \ref{g smooth} and in the last inequality we use Lemma \ref{PL to EB QG}. Plugging into (\ref{V dif}),
	\begin{equation*} 
		\mathbb{E} V_t - \mathbb{E} V_{t+1} 
		\geq  \frac{\tau_1}{32}\mathbb{E}\left\|\nabla \Phi\left(x_{t}\right)\right\|^{2} +  \kappa^2\tau_1 \mathbb{E}\|\nabla_y f(x_{t}, y_t)\|^2 - 292\kappa^4 l\tau_1^2\sigma^2.
	\end{equation*}
	By the same reasoning as the proof of Theorem \ref{thm agda} (after equation (\ref{V dif})), with the same stepsizes, we can show 
	\begin{align*}
		\frac{1}{T}\sum_{t=0}^{T-1}\mathbb{E}\left\|\nabla_x f\left(x_{t}, y_t\right)\right\|^{2} + &32\kappa^2\mathbb{E}\left\|\nabla_y f\left(x_{t}, y_t\right)\right\|^{2}\leq \\
		&\frac{d_0 l\kappa^2}{T}[\Phi(x_0)- \Phi^*] + \frac{d_1l\kappa^2}{T}a_0 + \frac{d_2\kappa^2\sqrt{l} a_0}{\sqrt{[\Phi(x_0)- \Phi^*]T}}\sigma + \frac{d_3\kappa^2\sqrt{l[\Phi(x_0)- \Phi^*]}}{\sqrt{T}}\sigma,
	\end{align*}
	where $d_0, d_1, d_2$ and $d_3$ are $O(1)$ constants.

\end{proof}

\section{Proofs for Stochastic Smoothed AGDA}
\label{apdx s-agda}

Before we present the theorem and converge, we adopt the following notations.

\begin{itemize}
	\item $\hat{f}(x, y; z) = f(x, y) + \frac{p}{2}\|x-z\|^2$: the auxiliary function;
	\item $\Psi(y;z) = \min_x \hat{f}(x, y; z)$: the dual function of the auxiliary problem;
	\item $\Phi(x; z) = \max_y \hat{f}(x, y; z)$: the primal function of the auxiliary problem;
	\item $P(z) = \min_x\max_y \hat{f}(x, y; z)$: the optimal value for the auxiliary function fixing $z$;
	\item $x^*(y,z) = \argmin_x \hat{f}(x, y; z)$: the optimal $x$ w.r.t $y$ and $z$ in the auxiliary function;
	\item $x^*(z) = \argmin_x \Phi(x;z)$: the optimal $x$ w.r.t $z$ in the auxiliary function when $y$ is already optimal w.r.t $x$;
	\item $Y^*(z) = \Argmax_y \Psi(y;z)$: the optimal set of $y$ w.r.t $z$ when $x$ is optimal to $y$;
	\item $y^+(z) = y+\tau_2\nabla_y \hat{f}(x^*(y,z), y;z)$: $y$ after one step of gradient ascent in $y$ with the gradient of the dual function;
	\item $x^+(y,z) = x-\tau_1\nabla_x\hat{f}(x,y;z)$: $x$ after one step of gradient descent with gradient at current point;
	\item $\hat{G}_x(x, y, \xi; z) = G_x(x, y, \xi) + p(x-z)$: the stochastic gradient for regularized auxiliary function. 
\end{itemize}

\begin{lemma} \label{acc helper lemma}
	We have the following inequalities as $p>l$
	\begin{align*}
		&\|x^*(y, z)- x^*(y, z^\prime)\| \leq \gamma_1 \|z-z^\prime\|, \\
		&\|x^*(z) - x^*(z^\prime) \leq \gamma_1\|z-z^\prime\|,\\
		&\|x^*(y,z) - x^*(y^\prime, z)\| \leq \gamma_2\|y - y^\prime\|,\\
		&\mathbb{E}\|x_{t+1} - x^*(y_t, z_t)\|^2 \leq \gamma_3^2\tau_1^2\mathbb{E}\|\nabla_x \hat{f}(x_t, y_t; z_t)\|^2 + 2\tau_1^2\sigma^2,
	\end{align*}
	where $\gamma_1 = \frac{p}{-l+p}$, $\gamma_2 = \frac{l+p}{-l+p}$ and $\gamma_3^2 = \frac{2}{\tau_1^2(-l+p)^2}+2$.
\end{lemma}

\begin{proof}
	The first and second inequality is the same as Proposition B.4 in \citep{zhang2020single}. The third inequality is a direct result of Lemma \ref{lin's lemma}. Now we show the last inequality. 
	\begin{align*}
		\|x_{t+1}-x^*(y_t, z_t)\|^2 \leq&  2\|x_t-x^*(y_t,z_t)\|^2 + 2\|x_{t+1}-x_t\|^2 \\
		\leq & \frac{2}{(-l+p)^2}\|\nabla_x \hat{f}(x_t, y_t; z_t)\|^2 + 2\tau_1^2\|\hat{G}_x(x_t, y_t, \xi_1^t; z_t)\|^2.
	\end{align*}
	where the second inequality use $(-l+p)$-strong convexity of $\hat{f}(\cdot, y_t;z_t)$. Taking expectation
	\begin{align*}
		\mathbb{E} \|x_{t+1}-x^*(y_t, z_t)\|^2 \leq& \frac{2}{(-l+p)^2}\mathbb{E}\|\nabla_x \hat{f}(x_t, y_t; z_t)\|^2 + 2\tau_1^2\mathbb{E}\|\nabla_x\hat{f}(x_t, y_t; z_t)\|^2 + 2\tau_1^2\sigma^2 \\
		\leq &2\left[ \frac{1}{(-l+p)^2} + \tau_1^2\right]\mathbb{E}\|\nabla_x \hat{f}(x_t, y_t; z_t)\|^2 +2\tau_1^2\sigma^2 .
	\end{align*}
\end{proof}

\begin{lemma} \label{acc agda eb lemma}
	The following inequality holds
	\begin{equation}
		\|x^*(z) - x^*(y^+(z),z)\|^2\leq \frac{1}{(p-l)\mu}\left(1+\tau_2l + \frac{\tau_2l(p+l)}{p-l} \right)^2\|\nabla_y \hat{f}(x^*(y,z),y;z)\|^2.
	\end{equation}
\end{lemma}

\begin{proof}
	By the $(p-l)$-strong convexity of $\Phi(\cdot; z)$, we have 
	\begin{align*}
		\|x^*(z) - x^*(y^+(z),z)\|^2 \leq & \frac{2}{p-l}\left[\Phi(x^*(y^+(z),z);z) - \Phi(x^*(z);z) \right] \\
		\leq & \frac{2}{p-l}\left[ \Phi(x^*(y^+(z),z);z) - \hat{f}(x^*(y^+(z),z), y^+(z);z) + \hat{f}(x^*(y^+(z),z), y^+(z);z) - \Phi(x^*(z);z)\right]\\
		\leq & \frac{1}{(p-l)\mu}\|\nabla_y\hat{f}(x^*(y^+(z),z), y^+(z);z)\|^2,
	\end{align*}
	where in the last inequality we use $\mu$-PL of $\hat{f}(x, \cdot;z)$ and $\hat{f}(x^*(y^+(z),z), y^+(z);z) \leq \Phi(x^*(z);z)$. Then 
	\begin{align*}
		\|\nabla_y\hat{f}(x^*(y^+(z),z), y^+(z);z)\| \leq & \|\nabla_y \hat{f}(x^*(y,z),y;z)\| + \|\nabla_y \hat{f}(x^*(y,z),y;z) - \nabla_y\hat{f}(x^*(y^+(z),z), y^+(z);z)\| \\
		\leq & \|\nabla_y \hat{f}(x^*(y,z),y;z)\| + l\|x^*(y,z)-x^*(y^+(z),z)\| + l\|y-y^+(z)\| \\
		\leq & \left(1+ \frac{\tau_2l(p+l)}{p-l} + \tau_2l\right)\|\nabla_y \hat{f}(x^*(y,z),y;z)\|,
	\end{align*}
	where in the last inequality we use Lemma \ref{acc helper lemma} and $\|y-y^+(z)\| =\tau_2\|\nabla_y \hat{f}(x^*(y,z),y;z)\| $. We reach our conclusion by combining with the previous inequality.
	
\end{proof}

\textbf{Proof of Theorem \ref{thm s-agda}}

\begin{proof}  The proof is built on \citep{zhang2020single}. We separate our proof into several parts: we first present three descent lemmas, then we show the descent property for a potential function, later we discuss the relation between our stationary measure and the potential function, and last we put things together.  
	\paragraph{Primal descent:} 
	
	By the $(p+l)$-smoothness of $\hat{f}(\cdot, y_t; z_t)$,
	\begin{align*}
		\hat{f}(x_{t+1}, y_t; z_t) \leq & \hat{f}(x_{t}, y_t; z_t) + \langle \nabla_x \hat{f}(x_{t}, y_t; z_t), x_{t+1} - x_t\rangle + \frac{p+l}{2}\|x_{t+1}-x_t\|^2 \\
		= & \hat{f}(x_{t}, y_t; z_t) - \tau_1 \langle \nabla_x \hat{f}(x_{t}, y_t; z_t), \hat{G}_x (x_t, y_t, \xi_1^t; z_t)\rangle + \frac{p+l}{2}\tau_1^2 \| \hat{G}_x (x_t, y_t, \xi_1^t; z_t)\|^2,
	\end{align*}
	We can easily verify that $\mathbb{E}  \hat{G}_x (x_t, y_t, \xi_1^t; z_t) = \nabla_x \hat{f}(x_t, y_t; z_t)$, and $\mathbb{E}\|\hat{G}_x(x_t,y_t, \xi_1^t; z_t) - \mathbb{E} \hat{G}_x (x_t, y_t, \xi_1^t; z_t)\|^2 = \mathbb{E}\|G_x(x_t,y_t, \xi_1^t) - \nabla_x f(x_t, y_t)\|^2\leq \sigma^2$. Taking expectation of both sides, 
	\begin{align} \nonumber
		\mathbb{E}\hat{f}(x_{t+1}, y_t; z_t) \leq  \mathbb{E}\hat{f}(x_{t}, y_t; z_t) - \tau_1 \mathbb{E}\|\nabla_x \hat{f}(x_{t}, y_t; z_t)\|^2 + \frac{p+l}{2}\tau_1^2  \mathbb{E}\|\nabla_x \hat{f}(x_{t}, y_t; z_t)\|^2 + \frac{p+l}{2}\tau_1^2\sigma^2.
	\end{align}
	As $\tau_1 \leq \frac{1}{p+l}$, 
	\begin{equation} \label{acc primal x}
		\mathbb{E}\hat{f}(x_{t}, y_t; z_t)- \mathbb{E}\hat{f}(x_{t+1}, y_t; z_t)  \geq \frac{\tau_1}{2}\mathbb{E}\|\nabla_x \hat{f}(x_{t}, y_t; z_t)\|^2 - \frac{p+l}{2}\tau_1^2\sigma^2.
	\end{equation}
	Also, because $\hat{f}(x_{t+1}, \cdot; z_t)$ is smooth, 
	\begin{align*}
		\hat{f}(x_{t+1}, y_t; z_t) - \hat{f}(x_{t+1}, y_{t+1}; z_t) \geq & \langle \nabla_y \hat{f}(x_{t+1}, y_t; z_t), y_{t}-y_{t+1}\rangle - \frac{l}{2}\|y_t-y_{t+1}\|^2 \\
		=& -\tau_2\langle \nabla_y  \hat{f}(x_{t+1}, y_t; z_t), G_y(x_{t+1}, y_t, \xi_2^t)\rangle - \frac{l}{2}\tau_2^2\|G_y(x_{t+1}, y_t, \xi_2^t)\|^2.
	\end{align*}
	Taking expectation of both sides,
	\begin{align} \nonumber
		\mathbb{E}\hat{f}(x_{t+1}, y_t; z_t) - \mathbb{E}\hat{f}(x_{t+1}, y_{t+1}; z_t) \geq& -\tau_2\mathbb{E}\|\nabla_y f(x_{t+1}, y_t)\|^2 - \frac{l}{2}\tau_2^2\mathbb{E}\|\nabla_y f(x_{t+1}, y_t)\|^2 - \frac{l}{2}\tau_2^2\sigma^2 \\ \label{acc primal y}
		=& -\left(1+\frac{l\tau_2}{2}\right)\tau_2\mathbb{E}\|\nabla_y f(x_{t+1},y_t)\|^2 - \frac{l}{2}\tau_2^2\sigma^2.
	\end{align}
	Furthermore, by definition of $\hat{f}$ and $z_{t+1}$, as $0<\beta<1$
	\begin{align} \nonumber
		\hat{f}(x_{t+1}, y_{t+1}; z_t) - \hat{f}(x_{t+1}, y_{t+1}; z_{t+1}) = & \frac{p}{2}[\|x_{t+1}-z_t\|^2 - \|x_{t+1}-z_{t+1}\|^2]  = \frac{p}{2}\left[\frac{1}{\beta^2}\|(z_{t+1} - z_t)\|^2 - \|(1-\beta)(x_{t+1}-z_t)\|^2 \right] \\ \label{acc primal z}
		=& \frac{p}{2}\left[\frac{1}{\beta^2}\|z_{t+1} - z_t\|^2 - \frac{(1-\beta)^2}{\beta^2}\|z_{t+1}-z_t\|^2 \right] \geq \frac{p}{2\beta}\|z_t-z_{t+1}\|^2.
	\end{align}
	Combining (\ref{acc primal x}), (\ref{acc primal y}) and (\ref{acc primal z}),
	\begin{align} \nonumber
		\mathbb{E}\hat{f}(x_{t}&, y_t; z_t)- \mathbb{E}\hat{f}(x_{t+1}, y_t; z_t)  \geq \\ &\frac{\tau_1}{2}\mathbb{E}\|\nabla_x \hat{f}(x_{t}, y_t; z_t)\|^2 -\left(1+\frac{l\tau_2}{2}\right)\tau_2\mathbb{E}\|\nabla_y f(x_{t+1},y_t)\|^2  + \frac{p}{2\beta}\mathbb{E}\|z_t-z_{t+1}\|^2 - \frac{l}{2}\tau_2^2\sigma^2 - \frac{p+l}{2}\tau_1^2\sigma^2.
	\end{align}
	
	\paragraph{Dual Descent:}
	Since the dual function $\Psi(y;z)$ is $L_{\Psi}$ smooth with $L_{\Psi} = l+l\gamma_2$ by Lemma B.3 in \citep{zhang2020single},
	\begin{align*} \nonumber
		\Psi(y_{t+1};z_t) - \Psi(y_t;z_t) \geq &\langle \nabla_y \Psi(y_t;z_t), y_{t+1}-y_t\rangle - \frac{L_{\Psi}}{2}\|y_{t+1}-y_t\|^2 \\ 
		= & \langle \nabla_y \hat{f}(x^*(y_t,z_t), y_t;z_t), y_{t+1}-y_t\rangle - \frac{L_{\Psi}}{2}\|y_{t+1}-y_t\|^2.
	\end{align*}
	Taking expectation, 
	\begin{equation} \label{acc dual y}
		\mathbb{E}\Psi(y_{t+1};z_t) - \mathbb{E}\Psi(y_t;z_t) \geq \tau_2\mathbb{E}\langle \nabla_y \hat{f}(x^*(y_t,z_t), y_t;z_t), \nabla_y f(x_{t+1}, y_t)\rangle - \frac{L_{\Psi}}{2}\tau_2^2\mathbb{E}\|\nabla_y f(x_{t+1}, y_t)\|^2 - \frac{L_{\Psi}}{2}\tau_2^2\sigma^2.
	\end{equation}
	Also, 
	\begin{align} \nonumber
		\Psi(y_{t+1};z_{t+1}) - \Psi(y_{t+1};z_t) =& \hat{f}(x^*(x_{t+1},z_{t+1}), y_{t+1}; z_{t+1}) - \hat{f}(x^*(y_{t+1},z_t), y_{t+1};z_t)  \\ \nonumber
		\geq & \hat{f}(x^*(x_{t+1},z_{t+1}), y_{t+1}; z_{t+1}) - \hat{f}(x^*(y_{t+1},z_{t+1}), y_{t+1};z_t) \\ \nonumber
		= & \frac{p}{2}\left[\|z_{t+1}-x^*(y_{t+1},z_{t+1})\|^2 - \|z_{t}-x^*(y_{t+1},z_{t+1})\|^2\right] \\ 
		=& \frac{p}{2}(z_{t+1}-z_t)^\top [z_{t+1}+z_t - 2x^*(y_{t+1},z_{t+1})].
	\end{align}
	Combining with (\ref{acc dual y}), we have
	\begin{align} \nonumber
		\mathbb{E}\Psi(y_{t+1};z_{t+1}) - \mathbb{E}\Psi(y_t;z_t) \geq & \tau_2\mathbb{E}\langle \nabla_y \hat{f}(x^*(y_t,z_t), y_t;z_t), \nabla_y f(x_{t+1}, y_t)\rangle - \frac{L_{\Psi}}{2}\tau_2^2\mathbb{E}\|\nabla_y f(x_{t+1}, y_t)\|^2 + \\
		& \frac{p}{2}\mathbb{E}(z_{t+1}-z_t)^\top [z_{t+1}+z_t - 2x^*(y_{t+1},z_{t+1})] - \frac{L_{\Psi}}{2}\tau_2^2\sigma^2.
	\end{align}
	
	\paragraph{Proximal Descent:} for all $y^*(z_{t+1})\in Y^*(z_{t+1})$ and $y^*(z_t) \in Y^*(z_t)$,
	\begin{align} \nonumber
		P(z_{t+1}) - P(z_t) = & \Psi(y^*(z_{t+1});z_{t+1}) - \Psi(y^*(z_{t});z_{t}) \\ \nonumber
		\leq & \Psi(y^*(z_{t+1});z_{t+1}) - \Psi(y^*(z_{t+1});z_{t}) \\ \nonumber
		= & \hat{f}(x^*(y^*(z_{t+1}), z_{t+1}),y^*(z_{t+1});z_{t+1}) - \hat{f}(x^*(y^*(z_{t+1}),z_t),y^*(z_{t+1});z_{t}) \\ \nonumber
		\leq & \hat{f}(x^*(y^*(z_{t+1}), z_{t}),y^*(z_{t+1});z_{t+1}) - \hat{f}(x^*(y^*(z_{t+1}),z_t),y^*(z_{t+1});z_{t}) \\ 
		= &\frac{p}{2}(z_{t+1}-z_t)^\top[z_{t+1}-z_t -2x^*(y^*(z_{t+1}),z_t)].
	\end{align}
	
	\paragraph{Potential Function} We use the potential function $ V _t =  V (x_t, y_t, z_t) = \hat{f}(x_t, y_t;z_t) - 2\Psi(y_t; z_t) + 2P(z_t)$. By three descent steps above, we have
	
	\begin{align} \nonumber
		\mathbb{E} V _t -\mathbb{E} V _{t+1} \geq & \frac{\tau_1}{2}\mathbb{E}\|\nabla_x \hat{f}(x_{t}, y_t; z_t)\|^2 -\left(1+\frac{l\tau_2}{2}\right)\tau_2\mathbb{E}\|\nabla_y f(x_{t+1},y_t)\|^2  + \frac{p}{2\beta}\mathbb{E}\|z_t-z_{t+1}\|^2 + \\ \nonumber
		&2\tau_2\mathbb{E}\langle \nabla_y \hat{f}(x^*(y_t,z_t), y_t;z_t), \nabla_y f(x_{t+1}, y_t)\rangle - L_{\Psi}\tau_2^2\mathbb{E}\|\nabla_y f(x_{t+1}, y_t)\|^2 + \\ \nonumber
		& p\mathbb{E}(z_{t+1}-z_t)^\top [z_{t+1}+z_t - 2x^*(y_{t+1},z_{t+1})] - p\mathbb{E}(z_{t+1}-z_t)^\top[z_{t+1}-z_t -2x^*(y^*(z_{t+1}),z_t)] -  \\ \nonumber
		&\frac{l}{2}\tau_2^2\sigma^2 - \frac{p+l}{2}\tau_1^2\sigma^2 - L_{\Psi}\tau_2^2\sigma^2 \\ \nonumber
		\geq & \frac{\tau_1}{2}\mathbb{E}\|\nabla_x \hat{f}(x_{t}, y_t; z_t)\|^2 + \left(1-\frac{l\tau_2}{2}-L_{\Psi}\tau_2\right)\tau_2\mathbb{E}\|\nabla_y f(x_{t+1},y_t)\|^2  + \frac{p}{2\beta}\mathbb{E}\|z_t-z_{t+1}\|^2 + \\ \nonumber
		& 2\tau_2\mathbb{E}\langle \nabla_y \hat{f}(x^*(y_t,z_t), y_t;z_t) - \nabla_yf(x_{t+1}, y_t), \nabla_y f(x_{t+1}, y_t)\rangle + \\ \nonumber
		&p\mathbb{E}(z_{t+1}-z_t)^\top [2x^*(y^*(z_{t+1}),z_t)- 2x^*(y_{t+1},z_{t+1})] - \frac{l}{2}\tau_2^2\sigma^2 - \frac{p+l}{2}\tau_1^2\sigma^2 - L_{\Psi}\tau_2^2\sigma^2 \\ \nonumber
		\geq & \frac{\tau_1}{2}\mathbb{E}\|\nabla_x \hat{f}(x_{t}, y_t; z_t)\|^2 + \frac{\tau_2}{2}\mathbb{E}\|\nabla_y f(x_{t+1},y_t)\|^2  + \frac{p}{2\beta}\mathbb{E}\|z_t-z_{t+1}\|^2 + \\ \nonumber
		& 2\tau_2\mathbb{E}\langle \nabla_y \hat{f}(x^*(y_t,z_t), y_t;z_t) - \nabla_yf(x_{t+1}, y_t), \nabla_y f(x_{t+1}, y_t)\rangle + \\ \label{acc potential}
		&2p\mathbb{E}(z_{t+1}-z_t)^\top [x^*(y^*(z_{t+1}),z_t)- x^*(y_{t+1},z_{t+1})] - \frac{l}{2}\tau_2^2\sigma^2 - \frac{p+l}{2}\tau_1^2\sigma^2 - L_{\Psi}\tau_2^2\sigma^2,
	\end{align}
	where in the last inequality we use $1-\frac{l\tau_2}{2}-L_{\Psi}\tau_2\geq\frac{1}{2}$ since $L_{\Psi} = 4l$ by our choice of $\tau_2$ and $p$. Now we denote $A = 2\tau_2\langle \nabla_y \hat{f}(x^*(y_t,z_t), y_t;z_t) - \nabla_yf(x_{t+1}, y_t), \nabla_y f(x_{t+1}, y_t)\rangle$ and $B = 2p(z_{t+1}-z_t)^\top [x^*(y^*(z_{t+1}),z_t)- x^*(y_{t+1},z_{t+1})]$.
	\begin{align} \label{acc B} \nonumber
		B =  & 2p(z_{t+1}-z_t)^\top [x^*(y^*(z_{t+1}),z_t)- x^*(y^*(z_{t+1}),z_{t+1})] + 2p(z_{t+1}-z_t)^\top [x^*(y^*(z_{t+1}),z_{t+1})- x^*(y_{t+1},z_{t+1})] \\ \nonumber
		\geq & -2p\gamma_1 \|z_{t+1}-z_t\|^2 + 2p(z_{t+1}-z_t)^\top [x^*(y^*(z_{t+1}),z_{t+1})- x^*(y_{t+1},z_{t+1})] \\
		\geq &
		-\left(2p\gamma_1 + \frac{p}{6\beta} \right)\|z_{t+1}-z_t\|^2 - 6p\beta\|x^*(y^*(z_{t+1}), z_{t+1})-x^*(y_{t+1}, z_{t+1})\|^2,
	\end{align}
	where we use \ref{acc helper lemma} in the first inequality. Also, 
	\begin{align} \nonumber
		A \geq& -2\tau_2\|\nabla_y \hat{f}(x^*(y_t,z_t), y_t;z_t) - \nabla_yf(x_{t+1}, y_t)\|\|\nabla_y f(x_{t+1}, y_t)\| \\ \nonumber
		\geq & -2\tau_2l\|x_{t+1}-x^*(y_t,z_t)\|\|\nabla_y f(x_{t+1}, y_t)\| \\
		\geq & -\tau_2^2l\nu \|\nabla_y f(x_{t+1},y_t)\|^2 - l\nu^{-1}\|x_{t+1}-x^*(y_t,z_t)\|^2,
	\end{align}
	where in the second inequality we use $\nabla_y \hat{f}(x^*(y_t,z_t), y_t;z_t) = \nabla_yf(x^*(y_t,z_t), y_t)$ and in the third inequality $\nu>0$ and we will choose it later. Taking expectation and applying Lemma \ref{acc helper lemma}
	\begin{equation} \label{acc A}
		\mathbb{E}A \geq  -\tau_2^2l\nu \mathbb{E}\|\nabla_y f(x_{t+1},y_t)\|^2 - l\tau_1^2\nu^{-1}\gamma_3^2\mathbb{E}\|\nabla_x \hat{f}(x_t, y_t; z_t)\|^2 - 2l\nu^{-1}\tau_1^2\sigma^2.
	\end{equation}
	Plugging (\ref{acc A}) and (\ref{acc B}) into (\ref{acc potential}),
	\begin{align} \nonumber
		\mathbb{E} V _t -\mathbb{E} V _{t+1} \geq &  \left(\frac{\tau_1}{2}-l\tau_1^2\nu^{-1}\gamma_3^2\right)\mathbb{E}\|\nabla_x \hat{f}(x_{t}, y_t; z_t)\|^2 + \left(\frac{\tau_2}{2}-\tau_2^2l\nu\right)\mathbb{E}\|\nabla_y f(x_{t+1},y_t)\|^2  + \\ \nonumber
		& \left(\frac{p}{2\beta}-2p\gamma_1 - \frac{p}{6\beta}\right)\mathbb{E}\|z_t-z_{t+1}\|^2 - 
		6p\beta\mathbb{E}\|x^*(y^*(z_{t+1}),z_{t+1}) - x^*(y_{t+1}, z_{t+1})\|^2 -  \\ \label{acc potential 2}
		& \left(\frac{p+l}{2}+2l\nu^{-1}\right)\tau_1^2\sigma^2 - \left(\frac{l}{2}+L_{\Psi}\right)\tau_2^2\sigma^2,
	\end{align}
	We rewrite $\|\nabla_y f(x_{t+1}, y_t)\|^2$ as:
	\begin{align} \nonumber
		\|\nabla_y f(x_{t+1}, y_t)\|^2 =& \|\nabla_y \hat{f}(x^*(y_t,z_t), y_t; z_t) + \nabla_y f(x_{t+1}, y_t) - \nabla_y \hat{f}(x^*(y_t,z_t), y_t; z_t)\|^2 \\ \nonumber
		\geq & \|\nabla_y \hat{f}(x^*(y_t,z_t), y_t; z_t)\|^2/2 - \| \nabla_y f(x_{t+1}, y_t) - \nabla_y \hat{f}(x^*(y_t,z_t), y_t; z_t)\|^2 \\
		\geq & \|\nabla_y \hat{f}(x^*(y_t,z_t), y_t; z_t)\|^2/2 - l^2\|x_{t+1} - x^*(y_t, z_t)\|^2.
	\end{align}
	Taking expectation and applying Lemma \ref{acc helper lemma}
	\begin{equation} \label{acc convert gradient y}
		\mathbb{E} \|\nabla_y f(x_{t+1}, y_t)\|^2 \geq \mathbb{E}\|\nabla_y \hat{f}(x^*(y_t,z_t), y_t; z_t)\|^2/2 - l^2\gamma_3^2\tau_1^2\mathbb{E}\|\nabla_x \hat{f}(x_t, y_t; z_t)\|^2 - 2l^2\tau_1^2\sigma^2.
	\end{equation}
	Note that $x^*(y^*(z_{t+1}),z_{t+1}) = x^*(z_{t+1})$. We rewrite $\|x^*(y^*(z_{t+1}),z_{t+1}) - x^*(y_{t+1}, z_{t+1})\|^2$ as
	\begin{align*}
		\| x^*(z_{t+1}) - x^*(y_{t+1}, z_{t+1})\|^2 \leq & 4\|x^*(z_{t+1}) - x^*(z_t)\|^2 + 4\|x^*(z_t) - x^*(y_t^+(z_t),z_t)\|^2 + \\
		& 4\|x^*(y_t^+(z_t),z_t) - x^*(y_{t+1},z_t)\|^2 + 4\|x^*(y_{t+1},z_t) - x^*(y_{t+1}, z_{t+1})\|^2 \\
		\leq & 4\gamma_1^2\|z_{t+1}-z_t\|^2 + 4\|x^*(z_t) - x^*(y_t^+(z_t),z_t)\|^2 + 4 \gamma_2^2\|y_t^+(z_t)-y_{t+1}\|^2 + 4\gamma_1^2\|z_t- z_{t+1}\|^2 \\
		\leq&  4\|x^*(z_t) - x^*(y_t^+(z_t),z_t)\|^2 + 8\gamma_2^2\tau_2^2\|\nabla_y \hat{f}(x^*(y_t),z_t),y_t;z_t)-\nabla_yf(x_{t+1}, y_t)\|^2+ \\
		&8\gamma_2^2\tau_2^2\|\nabla_yf(x_{t+1}, y_t) - G_y(x_{t+1},y_t, \xi_2^t)\|^2 + 8\gamma_1^2\|z_t- z_{t+1}\|^2 \\ 
		\leq & 4\|x^*(z_t) - x^*(y_t^+(z_t),z_t)\|^2 + 8\gamma_2^2\tau_2^2l^2\|x^*(y_t) - x_{t+1}\|^2+ \\
		&8\gamma_2^2\tau_2^2\|\nabla_yf(x_{t+1}, y_t) - G_y(x_{t+1},y_t, \xi_2^t)\|^2 + 8\gamma_1^2\|z_t- z_{t+1}\|^2,
	\end{align*}
	where in the second and last inequality we use Lemma \ref{acc helper lemma}, and in the third inequality we use the definition of $y_t^+(z_t)$. Taking expectation and applying Lemma \ref{acc helper lemma}
	\begin{align} \nonumber
		\mathbb{E} \| x^*(z_{t+1}) - x^*(y_{t+1}, z_{t+1})\|^2 \leq& 8\gamma_1^2\mathbb{E}\|z_t- z_{t+1}\|^2 + 4\mathbb{E}\|x^*(z_t) - x^*(y_t^+(z_t),z_t)\|^2 +\\ \label{acc convert x star}
		& 8\gamma_2^2\tau_2^2l^2\gamma_3^2\tau_1^2\mathbb{E}\|\nabla_x \hat{f}(x_t, y_t;z_t)\|^2 +
		16\gamma_2^2\tau_2^2l^2\tau_1^2\sigma^2 + 8\gamma_2^2\tau_2^2\sigma^2.
	\end{align}
	Plugging (\ref{acc convert x star}) and (\ref{acc convert gradient y}) into (\ref{acc potential 2}), we have 
	\begin{align} \nonumber
		&\mathbb{E} V _t -\mathbb{E} V _{t+1}\\ \nonumber
		\geq &   \left[\frac{\tau_1}{2}-l\tau_1^2\nu^{-1}\gamma_3^2 -  \left(\frac{\tau_2}{2}-\tau_2^2l\nu\right) l^2\gamma_3^2\tau_1^2 - 48p\beta\gamma_2^2\tau_2^2l^2\gamma_3^2\tau_1^2 \right]\mathbb{E}\|\nabla_x \hat{f}(x_{t}, y_t; z_t)\|^2 - 24p\beta\mathbb{E}\|x^*(z_t) - x^*(y_t^+(z_t),z_t)\|^2+\\ \nonumber
		& \left(\frac{\tau_2}{4}-\frac{\tau_2^2l\nu}{2}\right)\mathbb{E}\|\nabla_y \hat{f}(x^*(y_t,z_t), y_t; z_t)\|^2  + \left[\frac{p}{2\beta}-2p\gamma_1 - \frac{p}{6\beta} - 48p\beta\gamma_1^2\right]\mathbb{E}\|z_t-z_{t+1}\|^2 -  \\ \nonumber
		& \left[\frac{p+l}{2}+2l\nu^{-1} + 96p\beta\gamma_2^2\tau_2^2l^2 + 2l^2 \left(\frac{\tau_2}{2}-\tau_2^2l\nu\right)\right]\tau_1^2\sigma^2 - \left[\frac{l}{2}+L_{\Psi} + 48p\beta\gamma_2^2\right]\tau_2^2\sigma^2  \\  \nonumber
		\geq & \frac{\tau_1}{4}\mathbb{E}\|\nabla_x \hat{f}(x_{t}, y_t; z_t)\|^2  + \frac{\tau_2}{8}\mathbb{E}\|\nabla_y \hat{f}(x^*(y_t,z_t), y_t; z_t)\|^2 + \frac{p}{4\beta}\mathbb{E}\|z_t-z_{t+1}\|^2 - \\ \label{acc potential final bdn}
		&24p\beta\mathbb{E}\|x^*(z_t) - x^*(y_t^+(z_t),z_t)\|^2 - 2l\tau_1^2\sigma^2 - 5l\tau_2^2\sigma^2,
	\end{align}
	where in the last inequality we note that by our choice of $\tau_1, \tau_2, p$ and $\beta$ we have $\gamma_1 = 2$, $\gamma_2 = 3$ and $\gamma_3 = \frac{2}{\tau_1^2l^2}+2$ and therefore as we choose $\nu = \frac{1}{4l\tau_2} = \frac{12}{l\tau_1}$ we have $\frac{\tau_2}{4}-\frac{\tau_2^2l\nu}{2} = \frac{\tau_2}{8}$ and 
	\begin{align*}
		l\tau_1^2\nu^{-1}\gamma_3^2 + \left(\frac{\tau_2}{2}-\tau_2^2l\nu\right) l^2\gamma_3^2\tau_1^2 + 48p\beta\gamma_2^2\tau_2^2l^2\gamma_3^2\tau_1^2 =& \left[\nu^{-1}(l\tau_1\gamma_3^2) - \frac{1}{\tau_1}\frac{\tau_2}{4}(l^2\tau_1^2\gamma_3^2)+ 486l\beta\frac{\tau_2^2}{\tau_1}(l^2\tau_1^2\gamma_3^2) \right]\tau_1 \\
		\leq & \left[2\nu^{-1}\left(\frac{1}{\tau_1l}+\tau_1l \right) +  \frac{1}{96}\left(1+\tau_1^2l^2 \right) + \frac{486\times2}{48\times1600}l\mu\tau_2^2\left(1+\tau_1^2l^2 \right)\right]\tau_1 \\
		\leq & \left[\frac{20}{9\nu}\frac{1}{\tau_1l} + \frac{1}{96}\left( 1+ \frac{1}{9}\right) + \frac{486\times2}{48\times1600}\left(1+\frac{1}{9} \right)l\mu\tau_2^2 \right]\tau_1 \leq \frac{\tau_1}{4},
	\end{align*}
	and 
	\begin{align*}
		\frac{p+l}{2}+2l\nu^{-1} + 96p\beta\gamma_2^2\tau_2^2l^2 + 2l^2 \left(\frac{\tau_2}{2}-\tau_2^2l\nu\right) \leq \left[\frac{3}{2} + \frac{\tau_1l}{12} + \frac{96\times2\times9}{1600}l^2\mu\tau_2^3 + \frac{\tau_2l}{2} \right]l \leq  2l,
	\end{align*}
	and 
	\begin{align*}
		\frac{l}{2}+L_{\Psi} + 48p\beta\gamma_2^2 \leq \left[\frac{1}{2}+4 + 48\times2\times4\times9\beta \right]l \leq 5l,
	\end{align*}
	and 
	\begin{align*}
		\frac{p}{2\beta}-2p\gamma_1 - \frac{p}{6\beta} - 48p\beta\gamma_1^2 \geq & \left[\frac{1}{3} - 4\beta - 192\beta^2  \right]\frac{p}{\beta} \geq \frac{p}{4\beta}.
	\end{align*}

	\paragraph{Stationary Measure:} First we note that 
	\begin{align*}
		\|\nabla_x f(x_t, y_t)\| \leq \|\nabla_x \hat{f}(x_t, y_t; z_t)\| + p\|x_t-z_t\| \leq & \|\nabla_x \hat{f}(x_t, y_t; z_t) \| + p \|x_t-x_{t+1}\| + p\|x_{t+1}-z_t\| \\
		\leq &\|\nabla_x \hat{f}(x_t, y_t; z_t) \| + p\tau_1 \|\hat{G}_x(x_t, y_t, \xi_1^t;z_t)\| + p\|x_{t+1}-z_t\|.
	\end{align*}
	Taking square and expectation
	\begin{align} \nonumber
		\mathbb{E}\|\nabla_x f(x_t, y_t)\|^2 \leq& 6\mathbb{E}\|\nabla_x \hat{f}(x_t, y_t; z_t) \|^2 + 6p^2\tau_1^2\mathbb{E} \|\nabla_x\hat{f}(x_t, y_t;z_t)\|^2 + 6p^2\mathbb{E}\|x_{t+1}-z_t\|^2 + 6p^2\tau_1^2\sigma^2 \\ \label{acc rewrite x gradient}
		= & 6(1+p^2\tau_1^2)\mathbb{E}\|\nabla_x \hat{f}(x_t, y_t; z_t) \|^2 + 6p^2\mathbb{E}\|x_{t+1}-z_t\|^2 + 6p^2\tau_1^2\sigma^2 .
	\end{align}
	Also,
	\begin{align*}
		\|\nabla_y f(x_t, y_t)\| \leq& \|\nabla_y f(x_{t+1}, y_t)\| + \|\nabla_y f(x_t, y_t) - \nabla_y f(x_{t+1}, y_t)\|\\
		\leq & \|\nabla_y f(x_{t+1}, y_t)\| + l\|x_{t+1}-x_t\| \\
		\leq & l\tau_1\|\hat{G}_x(x_t, y_t, \xi_1^t;z_t) \| + \|\nabla_y \hat{f}(x^*(y_t,z_t),y_t;z_t)\| + \|\nabla_y \hat{f}(x^*(y_t,z_t),y_t;z_t) - \nabla_y f(x_{t+1}, y_t)\| \\ 
		\leq &  l\tau_1\|\hat{G}_x(x_t, y_t, \xi_1^t;z_t) \| + \|\nabla_y \hat{f}(x^*(y_t,z_t),y_t;z_t)\| + l\|x_{t+1}-x^*(y_t,z_t)\|.
	\end{align*}
	Taking square, taking expectation and applying Lemma \ref{acc helper lemma}
	\begin{align} \nonumber
		&\mathbb{E}\|\nabla_y f(x_t, y_t)\|^2 \\ \nonumber
		\leq & 6l^2\tau_1^2\mathbb{E}\|\nabla_x\hat{f}(x_t, y_t;z_t) \|^2 + 6l^2\tau_1^2\sigma^2+ 6\mathbb{E}\|\nabla_y \hat{f}(x^*(y_t,z_t),y_t;z_t)\|^2  + 6l^2\gamma_3^2\tau_1^2\mathbb{E}\|\nabla_x \hat{f}(x_t, y_t;z_t)\|^2 + 12l^2\tau_1^2\sigma^2 \\ \label{acc rewrite y gradient}
		\leq & 6l^2\tau_1^2(1+\gamma_3^2)\mathbb{E}\|\nabla_x\hat{f}(x_t, y_t;z_t) \|^2 +6\mathbb{E}\|\nabla_y \hat{f}(x^*(y_t,z_t),y_t;z_t)\|^2  + 18l^2\tau_1^2\sigma^2.
	\end{align}
	Combining with (\ref{acc rewrite x gradient}),
	\begin{align} \nonumber
		&\mathbb{E}\|\nabla_x f(x_t, y_t)\|^2 +  \kappa \mathbb{E}\|\nabla_y f(x_t, y_t)\|^2 \\ \nonumber
		\leq & 6(1+p^2\tau_1^2+\kappa l^2\tau_1^2+\kappa l^2\gamma_3^2\tau_1^2)\mathbb{E}\|\nabla_x\hat{f}(x_t, y_t;z_t) \|^2 + 6\kappa\mathbb{E}\|\nabla_y \hat{f}(x^*(y_t,z_t),y_t;z_t)\|^2  + \\ \nonumber
		& 6p^2\mathbb{E}\|x_{t+1}-z_t\|^2 + (6p^2 + 18\kappa l^2)\tau_1^2\sigma^2 \\ \label{acc moreau bdn}
		\leq & 24\kappa\mathbb{E}\|\nabla_x\hat{f}(x_t, y_t;z_t) \|^2 + 6\kappa\mathbb{E}\|\nabla_y \hat{f}(x^*(y_t,z_t),y_t;z_t)\|^2v  + 6p^2\mathbb{E}\|x_{t+1}-z_t\|^2 + 42\kappa l^2\tau_1^2\sigma^2,
	\end{align}
	where in the last inequality we use $6p^2 + 18\kappa l^2 = 24l^2+18\kappa l^2\leq 42\kappa l^2$ and
	\begin{align*}
		1+p^2\tau_1^2+kl^2\tau_1^2+\kappa l^2\gamma_3^2\tau_1^2 = & 1 + 4l^2\tau_1^2 + \kappa l^2 \tau_1^2 + 2\kappa(1 + \tau_1^2l^2) \\
		\leq & \frac{13}{9} + 2\kappa + 3\kappa l^2\tau_1^2 \leq 4\kappa.
	\end{align*}

	\paragraph{Putting pieces together:}
	From Lemma \ref{acc agda eb lemma},
	
	\begin{align*}
		24p\beta\|x^*(z) - x^*(y^+(z),z)\|^2\leq &\frac{24p\beta}{(p-l)\mu}\left(1+\tau_2l + \frac{\tau_2l(p+l)}{p-l} \right)^2\|\nabla_y \hat{f}(x^*(y,z),y;z)\|^2\\
		\leq  & \frac{1}{16}\tau_2\|\nabla_y \hat{f}(x^*(y_t,z_t), y_t; z_t)\|^2,
	\end{align*}
	where in the second inequality we use 
	\begin{align*}
		\frac{24p\beta}{(p-l)\mu}\left(1+\tau_2l + \frac{\tau_2l(p+l)}{p-l} \right)^2 = & \frac{48\beta}{\mu}\left(1+\tau_2l+3\tau_2l\right)^2 \leq \frac{96\beta}{\mu} \leq \frac{1}{16}\tau_2.
	\end{align*}
	Plugging into (\ref{acc potential final bdn}),
	\begin{align*}
		\mathbb{E} V _t -\mathbb{E} V _{t+1} \geq & \frac{\tau_1}{4}\mathbb{E}\|\nabla_x \hat{f}(x_{t}, y_t; z_t)\|^2  + \frac{\tau_2}{16}\mathbb{E}\|\nabla_y \hat{f}(x^*(y_t,z_t), y_t; z_t)\|^2 + \frac{p\beta}{4}\mathbb{E}\|z_t-x_{t+1}\|^2 - 2l\tau_1^2\sigma^2 - 5l\tau_2^2\sigma^2.
	\end{align*}
	Plugging into (\ref{acc moreau bdn}), 
	\begin{align} \nonumber
		&\mathbb{E}\|\nabla_x f(x_t, y_t)\|^2 +  \kappa \mathbb{E}\|\nabla_y f(x_t, y_t)\|^2   \\ \nonumber
		\leq &  24\kappa\mathbb{E}\|\nabla_x\hat{f}_x(x_t, y_t;z_t) \|^2 + 6\kappa\mathbb{E}\|\nabla_y \hat{f}(x^*(y_t,z_t),y_t;z_t)\|  + 6p^2\mathbb{E}\|x_{t+1}-z_t\|^2 + 42\kappa l^2\tau_1^2\sigma^2 \\ \nonumber
		\leq & \max\bigg\{\frac{96\kappa}{\tau_1}, \frac{96\kappa}{\tau_2}, \frac{24p}{\beta} \bigg\} \left[ \mathbb{E} V _t -\mathbb{E} V _{t+1} +2l\tau_1^2\sigma^2 + 5l\tau_2^2\sigma^2 \right] + 42\kappa l^2\tau_1^2\sigma^2  \\  \nonumber
		\leq & \frac{O(1)\kappa}{\tau_2}[\mathbb{E} V _t -\mathbb{E} V _{t+1}] + \frac{O(1)\kappa l\tau_1^2}{\tau_2}\sigma^2 + O(1)\kappa l \tau_2\sigma^2 + O(1)\kappa l^2\tau_1^2\sigma^2 \\ \nonumber
		\leq & \frac{O(1)\kappa}{\tau_1}[\mathbb{E} V _t -\mathbb{E} V _{t+1}] + O(1)\kappa l\tau_1\sigma^2 + O(1)\kappa l^2\tau_1^2\sigma^2 \\
		\leq & \frac{O(1)\kappa}{\tau_1}[\mathbb{E} V _t -\mathbb{E} V _{t+1}] + O(1)\kappa l\tau_1\sigma^2,
	\end{align}
	where in the second and fourth inequality we use $\tau_1 = 48\tau_2$ and $p/\beta = 3200\kappa/\tau_2$. Telescoping,
	\begin{align} \nonumber
		\frac{1}{T}\sum_{t=0}^{T-1} \mathbb{E}\|\nabla_x f(x_t, y_t)\|^2 +  \kappa \mathbb{E}\|\nabla_y f(x_t, y_t)\|^2   \leq & \frac{O(1)\kappa}{T\tau_1}[ V _0 - \min_{x,y,z} V (x,y,z)] + O(1)\kappa l\tau_1\sigma^2.
	\end{align}
	Note that since for any $z$ we can find $x, y$ such that $ (\hat{f}(x,y;z) - \Psi(y;z)) + (P(z) - \Psi(y;z)) = 0$,
	\begin{align*}
		& V_0-\min_{x,y,z} V(x,y,z) \\
		= & P(z_0) + (\hat{f}(x_0,y_0;z_0) - \Psi(y_0;z_0)) + (P(z_0) - \Psi(y_0;z_0)) - \min_{x,y,z} [P(z) + (\hat{f}(x,y;z) - \Psi(y;z)) + (P(z) - \Psi(y;z))] \\
		\leq & (P(z_0) - \min_z P(z)) + (\hat{f}(x_0,y_0;z_0) - \Psi(y_0;z_0)) + (P(z_0) - h(y_0;z_0)).
	\end{align*}
	Note that for any $z$
	$$P(z) = \min_x\max_y f(x, y) + l\|x-z\|^2 = \min_x \Phi(x) + l\|x-z\|^2 = \Phi_{1/2l}(z) \leq \Phi(z),$$
	and $P(z) = \Phi_{1/2l}(z)$ also implies $\min_z P(z) = \min_x \Phi(x)$. Hence
	\begin{equation} \label{initial dist acc-agda}
		V_0-\min_{x,y,z} V(x,y,z) \leq (\Phi(z_0) - \min_x \Phi(x)) + (\hat{f}(x_0,y_0;z_0) - \Psi(y_0;z_0)) + (P(z_0) - \Psi(y_0;z_0)).
	\end{equation}
	With $b= (\hat{f}(x_0,y_0;z_0) - \Psi(y_0;z_0)) + (P(z_0) - \Psi(y_0;z_0))$, we write
	\begin{align} \nonumber
		\frac{1}{T}\sum_{t=0}^{T-1} \mathbb{E}\|\nabla_x f(x_t, y_t)\|^2 +  \kappa \mathbb{E}\|\nabla_y f(x_t, y_t)\|^2  
		\leq  \frac{O(1)\kappa}{T\tau_1}[\Delta + b] + O(1)\kappa l\tau_1\sigma^2.
	\end{align}
	with $\Delta = \Phi(z_0) - \Phi^*$. Picking $\tau_1 = \min\bigg\{\frac{\sqrt{\Phi(x_0)-\Phi^*}}{2\sigma\sqrt{Tl}}, \frac{1}{3l} \bigg\}$,
	\begin{align*}
		\frac{1}{T}\sum_{t=0}^{T-1}\mathbb{E}\|\nabla_x f(x_t, y_t)\|^2 +  \kappa \mathbb{E}\|\nabla_y f(x_t, y_t)\|^2   \leq& \max\bigg\{\frac{2\sigma\sqrt{Tl}}{\sqrt{\Delta }}, 3l \bigg\}\frac{O(1)\kappa}{T}[\Phi(z_0) - \Phi^* + b] + 
		\frac{O(1)\sqrt{\Delta }}{2\sigma\sqrt{Tl}} \cdot \kappa l\tau_1\sigma^2 \\
		\leq & \frac{O(1)\kappa}{T}[\Delta  + b] + \frac{O(1)\kappa\sqrt{l} b}{\sqrt{\Delta T}}\sigma + \frac{O(1)\kappa\sqrt{l\Delta }}{\sqrt{T}}\sigma.
	\end{align*}
	We reach our conclusion by noting that $b \leq 2\gap_{\hat{f}(\cdot, \cdot; z_0)}(x_t, y_t)$.

\end{proof}

\section{Catalyst-AGDA}
\label{apdx catalyst}

\begin{algorithm}[h] 
	\caption{Catalyst-AGDA}
	\begin{algorithmic}[1]
		\STATE Input: $(x_0,y_0)$, step sizes $\tau_1>0, \tau_2>0$.
		\FORALL{$t = 0,1,2,..., T-1$}
		\STATE Let $k=0$ and $x_0^0 = x_0$.
		\REPEAT
		\STATE $y^t_{k+1} = y^t_k + \tau_2 \nabla_y f(x^t_{k}, y^t_k)$
		\STATE $x^t_{k+1} =  x^t_k - \tau_1 [\nabla_x f(x^t_k, y^t_{k+1})+2l(x_k^t-x^t_0)]$
		\STATE $k = k+1$
		\UNTIL{$\gap_{\hat{f}_t}(x^t_k, y^t_k)\leq \beta \gap_{\hat{f}_t}(x^t_0, y^t_0)$ where $ \hat{f}_{t}(x,y)\triangleq  f(x,y) + l\Vert x - x^t_{0}\Vert^2$}
		\STATE $x^{t+1}_0 = x^t_{k+1}, \quad y^{t+1}_0 = y^t_{k+1}$
		\ENDFOR
		\STATE Output: $\Tilde{x}_{T}$, which is uniformly sampled  from $x^1_{0},...,x^T_{0}$
	\end{algorithmic} \label{catalyst agda}
\end{algorithm}

In this section,we present a new algorithm, called Catalyst-AGDA, in Algorithm \ref{catalyst agda}. It iteratively solves an augmented auxiliary problem similar to Smoothed-AGDA:
\begin{equation*}
	\hat{f}_{t}(x,y)\triangleq  f(x,y) + l\Vert x - x^t_{0}\Vert^2,
\end{equation*}
by AGDA with $y$ update first\footnote{We believe that updating $x$ first in the subroutine will lead to the same convergence property. For simplicity, we update $y$ first so that we can directly apply Theorem \ref{thm two-sided pl}.}. The stopping criterion for the inner-loop is 
$$\gap_{\hat{f}_t}(x^t_k, y^t_k)\leq \beta \gap_{\hat{f}_t}(x^t_0, y^t_0), $$
and we will specify $\beta$ later.  For Catalyst-AGDA, we only consider the deterministic case, in which we have the exact gradient of $f(\cdot, \cdot)$. 

In this section, we use $(x^t, y^t)$ as a shorthand for $(x^t_0, y^t_0)$. We denote $(\hat{x}^t, \hat{y}^t)$ with $\hat{y}^t \in \hat{Y}^t$ as the optimal solution to the auxiliary problem at $t$-th iteration: $\min_{x\in \mathbb{R}^{d_1}}\max_{y\in \mathbb{R}^{d_2}} \left[\hat{f}_{t}(x,y)\triangleq  f(x,y) + l\Vert x - x^t\Vert^2 \right]$. Define $\hat{\Phi}_t(x) = \max_y f(x,y) + l\|x-x^t\|^2 $. We use $Y^*(x)$ to denote the set $\Argmax_y f(x, y)$. In the following lemma, we show the convergence of the Moreau envelop $\|\nabla\Phi_{1/2l}(x)\|^2$ when we choose $\beta$ appropriately in the stopping criterion of the AGDA subroutine. 

\begin{lemma} \label{catalyst moreau complexity}
	Under Assumptions \ref{Lipscthitz gradient} and \ref{PL assumption}, define $\Delta = \Phi(x_0)-\Phi^*$, if we apply Catalyst-AGDA with $\beta = \frac{\mu^2}{4l^2}$ in the stopping criterion of the inner-loop, then we have
	\begin{align*}
		\sum_{t=0}^{T-1}\|\nabla\Phi_{1/2l}(x^t)\|^2 \leq \frac{35l}{2}\Delta+3la_0,
	\end{align*}
	where $a_0 := \Phi(x_0) - f(x_0, y_0)$. 
\end{lemma}

\begin{proof}
	Define $g_{t+1} = \gap_{\hat{f_t}}(x^{t+1}, y^{t+1})$. It is easy to observe that $\hat{x}^t = \prox_{\Phi/2l}(x^t)$. Define $\hat{\Phi}_t(x) = \max_y f(x,y) + l\|x-x^t\|^2 $.  By Lemma 4.3 in \citep{drusvyatskiy2019efficiency}, 
	\begin{align} \nonumber
		\Vert \nabla \Phi_{1/2l}(x^t)\Vert^2 =  4l^2 \Vert x^t - \hat{x}^t \Vert^2 \leq & 8l [\hat{\Phi}_t(x^t) - \hat{\Phi}_t(\prox_{\Phi/2l}(x^t))] \\ \nonumber
		\leq & 8l [\hat{\Phi}_t(x^t) -\hat{\Phi}_t(x^{t+1}) + b_{t+1}] \\ \nonumber
		= & 8l\big\{\Phi(x^t) - \left[\Phi(x^{t+1})+l\Vert x^{t+1}-x^t\Vert^2 \right]+b_{t+1}\big\}\\  \label{bound x to prox 3}
		\leq & 8l [\Phi(x^t) - \Phi(x^{t+1}) + g_{t+1}],
	\end{align}
	where in the first inequality we use $l$-strongly convexity of $\hat{\Phi}_t$. Because $\hat{f}$ is $3l$-smooth, $l$-strongly convex in $x$ and $\mu$-PL in $y$, its primal and dual function are $18l\kappa$ and $18l$ smooth, respectively, by Lemma \ref{g smooth}. Then we have 
	\begin{align} \nonumber
		\gap_{\hat{f}_t}(x^t, y^t) =& \max_y \hat{f}_t(x^t, y) - \min_x\max_y \hat{f}_t(x, y) + \min_x\max_y \hat{f}_t(x, y) - \min_x \hat{f}_t (x, y_t)\\
		\leq & 9l\kappa \|x^t-\hat{x}^t\|^2 + 9l\|y^t-\hat{y}^t\|^2,
	\end{align}
	for all $\hat{y}^t \in \hat{Y}^t$. For $t\geq 1$, by fixing $\hat{y}^{t-1}$ to be the projection of $y^t$ to $\hat{Y}^{t-1}$, there exists $\hat{y}^t \in \hat{Y}^t$ so that
	\begin{align*}
		\|y^t - \hat{y}^t\|^2 \leq& 2\|y^t - \hat{y}^{t-1}\|^2 + 2\|y^*(\hat{x}^{t-1})-y^*(\hat{x}^t)\|^2 \\
		\leq & 2\|y^t - \hat{y}^{t-1}\|^2 + 2\left(\frac{l}{\mu}\right)^2\|\hat{x}^t-\hat{x}^{t-1}\|^2\\
		\leq &  2\|y^t - \hat{y}^{t-1}\|^2 + 4\left(\frac{l}{\mu}\right)^2\|\hat{x}^t-x^t\|^2 + 4\left(\frac{l}{\mu}\right)^2\|x^t-\hat{x}^{t-1}\|^2 \\
		\leq &  \frac{8l}{\mu^2}g_{t} + 4\left(\frac{l}{\mu}\right)^2\|\hat{x}^t-x^t\|^2,
	\end{align*}
	where we use Lemma \ref{g smooth} in the second inequality, and strong-convexity and PL condition in the last inequality.	By our stopping criterion and $ \Vert \nabla \Phi_{1/2l}(x^t)\Vert^2 =  4l^2 \Vert x^t - \hat{x}^t\Vert^2$, for $t\geq 1$
	\begin{equation} \label{recur 1}
		g_{t+1} \leq  \beta \gap_{\hat{f}_t}(x^t, y^t) \leq 9l\kappa\beta \|x^t-\hat{x}^t\|^2 + 9l\beta\|y^t-\hat{y}^t\|^2\leq 72\kappa^2\beta g_{t} + \frac{12\kappa^2\beta}{l}\Vert \nabla \Phi_{1/2l}(x^t)\Vert^2.
	\end{equation}
	For $t=0$, by fixing $y^*(x^0)$ to be the projection of $y^0$ to $Y^*(x^0)$,
	\begin{equation}  \label{ncc y0 bound}
		\|y^0-\hat{y}^0\|^2\leq 2\|y^0-y^*(x^0)\|^2 + 2\|\hat{y}^0-y^*(x^0)\|^2\leq \frac{4}{\mu}a_0  + 2\kappa^2\|x^0-\hat{x}^0\|^2.
	\end{equation}
	Because $\Phi(x)+l\|x-x^0\|^2$ is $l$-strongly convex, we have
	\begin{equation*}
		\left(\Phi(\hat{x}^0) +l\|\hat{x}^0-x^0\|^2 \right) + \frac{l}{2}\|\hat{x}^0-x^0\|^2\leq \Phi(x^0) = \Phi^* + (\Phi(x^0)-\Phi^*) \leq \Phi(\hat{x}^0) + (\Phi(x^0)-\Phi^*).
	\end{equation*}
	This implies $\|\hat{x}^0-x^0\|^2\leq \frac{2}{3l}(\Phi(x^0)-\Phi^*)$. 
	Hence, by the stopping criterion,
	\begin{equation} \label{recur 2}
		g_{1} \leq  \beta \gap_{\hat{f}_0}(x^0, y^0) \leq 9l\kappa\beta \|x^0-\hat{x}^0\|^2 + 9l\beta\|y^0-\hat{y}^0\|^2\leq 18\kappa^2\beta \Delta + 36\kappa \beta a_0.
	\end{equation}
	Recursing (\ref{recur 1}) and (\ref{recur 2}),
	we have for $t\geq1$
	\begin{align*} 
		g_{t+1} \leq & (72\kappa^2\beta)^tg_1 + \frac{12\kappa^2\beta}{l}\sum_{k=1}^t(72\kappa^2\beta)^{t-k}\|\nabla\Phi_{1/2l}(x_k)\|^2 \\
		\leq & 18\kappa^2\beta  (72\kappa^2\beta)^t \Delta + 36\kappa \beta(72\kappa^2\beta)^t a_0 + \frac{12\kappa^2\beta}{l}\sum_{k=1}^t(72\kappa^2\beta)^{t-k}\|\nabla\Phi_{1/2l}(x_k)\|^2.
	\end{align*}
	Summing from $t=0$ to $T-1$,
	\begin{align} \nonumber
		\sum_{t=0}^{T-1}g_{t+1} &= \sum_{t=1}^{T-1}g_t + g_1\\ \nonumber
		&\leq 18\kappa^2\beta\sum_{t=0}^{T-1} (72\kappa^2\beta)^t\Delta +36\kappa \beta\sum_{t=0}^{T-1} (72\kappa^2\beta)^ta_0 + \frac{12\kappa^2\beta}{l}\sum_{t=1}^{T-1}\sum_{k=1}^t (72\kappa^2\beta)^{t-k}\|\nabla\Phi_{1/2l}(x_k)\|^2 \\ \label{rewrite b_t}
		&\leq \frac{18\kappa^2\beta}{1-72\kappa^2\beta}\Delta + \frac{36\kappa\beta}{1-72\kappa^2\beta}a_0 + \frac{12\kappa^2\beta}{l(1- 72\kappa^2\beta)}\sum_{t=1}^{T-1}\|\nabla\Phi_{1/2l}(x^t)\|^2,
	\end{align}
	where in the last inequality $\sum_{t=1}^{T-1}\sum_{k=1}^t (72\kappa^2\beta)^{t-k}\|\nabla\Phi_{1/2l}(x_k)\|^2 = \sum_{k=1}^{T-1}\sum_{t=k}^T (72\kappa^2\beta)^{t-k}\|\nabla\Phi_{1/2l}(x_k)\|^2 \leq \sum_{k=1}^{T-1}\frac{1}{1- (72\kappa^2\beta)}\|\nabla\Phi_{1/2l}(x_k)\|^2$.
	Now, by telescoping (\ref{bound x to prox 3}),
	\begin{equation*}
		\frac{1}{8l}\sum_{t=0}^{T-1}\|\nabla\Phi_{1/2l}(x^t)\|^2 \leq \Phi(x^0)-\Phi^* + \sum_{t=0}^{T-1}g_{t+1}.
	\end{equation*}
	Plugging (\ref{rewrite b_t}) in,
	\begin{equation}
		\left(\frac{1}{8l}- \frac{12\kappa^2\beta}{l(1- 72\kappa^2\beta)} \right)\sum_{t=0}^{T-1}\|\nabla\Phi_{1/2l}(x^t)\|^2  \leq \left(1+\frac{18\kappa^2\beta}{1-72\kappa^2\beta}\right)\Delta + \frac{36\kappa\beta}{1-72\kappa^2\beta}a_0 .
	\end{equation}
	With $\beta = \frac{1}{264\kappa^4}$, we have $\frac{\kappa^2\beta}{1-72\kappa^2\beta} \leq \frac{1}{192\kappa^2}$. Therefore,
	\begin{align*}
		\sum_{t=0}^{T-1}\|\nabla\Phi_{1/2l}(x^t)\|^2 \leq \frac{35l}{2}\Delta+3la_0.
	\end{align*}
	
\end{proof}		

\begin{theorem}
	Under Assumptions \ref{Lipscthitz gradient} and \ref{PL assumption}, if we apply Catalyst-AGDA with $\beta = \frac{1}{264\kappa^4}$ in the stopping criterion of the inner-loop, then the output from Algorithm \ref{catalyst agda} satisfies
	\begin{equation}
		\sum_{t=1}^T\|\nabla \Phi(x^t_0)\|^2 \leq \frac{1}{T}\sum_{t=1}^T \|\nabla\Phi(x^{t+1})\|^2 \leq \frac{19l}{T}\Delta + \frac{6l}{T}a_0
	\end{equation}
	which implies the outer-loop complexity of $O(l\Delta\epsilon^{-2})$. Furthermore, if we choose $\tau_1 = \frac{1}{3l}$ and $\tau_2 = \frac{1}{486l}$, it takes $K = O(\kappa\log(\kappa))$ inner-loop iterations to satisfy the stopping criterion. Therefore, the total complexity is $O(\kappa l\Delta\epsilon^{-2}\log \kappa)$.
\end{theorem}		

\begin{proof}
	We separate the proof into two parts: 1) outer-loop complexity 2) inner-loop convergence rate.
	
	\textbf{Outer-loop}: We still denote $g_{t+1} = \gap_{\hat{f_t}}(x^{t+1}, y^{t+1})$. First, note that 
	\begin{align} \nonumber
		\|\nabla\Phi(x^{t+1})\|^2 &\leq 2\|\nabla\Phi(x^{t+1})-\nabla \Phi(\hat{x}^t)  \|^2 + 2\|\nabla\Phi(\hat{x}^t)\|^2\\ \nonumber
		&\leq 2\left(\frac{2l^2}{\mu} \right)\|x^{t+1}-\hat{x}^t\|^2 + 2\|\nabla \Phi_{1/2l}(x^t)\|^2 \\ 
		&\leq \frac{16l^3}{\mu^2}g_{t+1} + 2\|\nabla \Phi_{1/2l}(x^t)\|^2.
	\end{align}
	where in the second inequality we use Lemma \ref{lin's lemma} and Lemma 4.3 in \citep{drusvyatskiy2019efficiency}. Summing from $t=0$ to $T-1$, we have
	\begin{equation} \label{moreau to primal convergence}
		\sum_{t=0}^{T-1} \|\nabla\Phi(x^{t+1})\|^2 \leq \frac{16l^3}{\mu^2}\sum_{t=0}^{T-1} g_{t+1} + 2\sum_{t=0}^{T-1} \|\nabla \Phi_{1/2l}(x^t)\|^2.
	\end{equation}
	Applying (\ref{rewrite b_t}), we have
	\begin{equation*}
		\sum_{t=0}^{T-1} \|\nabla\Phi(x^{t+1})\|^2 \leq \left[ \frac{16l^3}{\mu^2}\cdot\frac{12\kappa^2\beta}{l(1- 72\kappa^2\beta)}+2\right]\sum_{t=1}^{T-1}\|\nabla\Phi_{1/2l}(x^t)\|^2 + \frac{16l^3}{\mu^2}\cdot\frac{18\kappa^2\beta}{1-72\kappa^2\beta}\Delta + \frac{16l^3}{\mu^2}\cdot \frac{36\kappa\beta}{1-72\kappa^2\beta}a_0,
	\end{equation*}
	With $\beta = \frac{1}{264\kappa^4}$, we have
	\begin{equation*}
		\sum_{t=0}^{T-1} \|\nabla\Phi(x^{t+1})\|^2 \leq 3\sum_{t=1}^{T-1}\|\nabla\Phi_{1/2l}(x^t)\|^2 + \frac{3l}{2}\Delta + 3la_0.
	\end{equation*}
	Applying Lemma \ref{catalyst moreau complexity},
	\begin{align*}
		\frac{1}{T}\sum_{t=1}^T \|\nabla\Phi(x^{t+1})\|^2 \leq \frac{19l}{T}\Delta + \frac{6l}{T}a_0.
	\end{align*}

	\textbf{Inner-loop}: The objective of auxiliary problem $\min_x\max_y \hat{f}_{t}(x,y)\triangleq  f(x,y) + l\Vert x - x^t_{0}\Vert^2$ is $3l$-smooth and $(l,\mu)$-SC-PL. We denote the dual function of the auxiliary problem by $\hat{\Psi}^t(y) = \min_x \hat{f}_t(x,y)$. We also define 
	$$P_k^t \triangleq \left[\max_y\hat{\Psi}^t(y) - \hat{\Psi}^t(y^t_k)\right]+\frac{1}{10}\left[\hat{f}_t(x^t_k, y^t_k)-\hat{\Psi}^t(y^t_k)\right].$$
	By Theorem \ref{thm two-sided pl}, AGDA with stepsizes $\tau_1 = \frac{1}{3l}$ and $\tau_2 = \frac{l^2}{18(3l)^3} = \frac{1}{486l}$ satisfies 
	\begin{align} \nonumber
		P_k^t \leq \left(1-\frac{\mu}{972l} \right)^k P_0^t.
	\end{align}
	We denote $x^t_*(y) = \argmin_x \hat{f}_t(x, y)$. We note that 
	\begin{align} \nonumber
		\|x^t_k-\hat{x}^t\|^2 =& 2\|x^t_k - x^t_*(y^t_k)\|^2 + 2 \|x^t_*(y^t_k)-\hat{x}^t\|^2 \\ \nonumber
		= & 2\|x^t_k - x^t_*(y^t_k)\|^2 + 2 \|x^t_*(y^t_k)-x^t_*(\hat{y}^t)\|^2 \\ \nonumber
		\leq & \frac{4}{l}\left[\hat{f}_t(x^t_k, y^t_k)-\hat{\Psi}^t(y^t_k)\right] + 2\left(\frac{3l}{\mu}\right)^2\|y^t_k-\hat{y}^t\|^2 \\ \nonumber
		\leq & \frac{4}{l}\left[\hat{f}_t(x^t_k, y^t_k)-\hat{\Psi}^t(y^t_k)\right] + \frac{36l^2}{\mu^3}[\hat{\Psi}^t(\hat{y}^t) - \hat{\Psi}^t(y^t_k)] \\ 
		\leq & \left(\frac{40}{l} + \frac{36l^2}{\mu^3} \right) \left(1-\frac{\mu}{972l} \right)^k P^t_0,
	\end{align}
	where in the first inequality we use $l$-strong convexity of $\hat{f}_t(\cdot, y^t_k)$ and Lemma \ref{lin's lemma}, and in the second inequality we use $\mu$-PL of $\hat{\Psi}^t$ and Lemma \ref{PL to EB QG}. Since $\hat{\Phi}^t$ is smooth by Lemma \ref{g smooth},  
	\begin{equation}
		\hat{\Phi}^t(x^t_k) - \hat{\Phi}^t(\hat{x}^t) \leq \frac{2(3l)^2}{2\mu}\|x^t_k-\hat{x}^t\|^2 \leq \frac{9l^2}{\mu} \left(\frac{40}{l} + \frac{36l^2}{\mu^3} \right) \left(1-\frac{\mu}{972l} \right)^k P^t_0.
	\end{equation}
	Therefore, 
	\begin{align*}
		\gap_{\hat{f}_t}(x^t_k, y^t_k) = \hat{\Phi}^t(x^t_k) - \hat{\Phi}^t(\hat{x}^t) + \hat{\Psi}^t(\hat{y}^t) - \hat{\Psi}^t(y^t_k) \leq & \left[\frac{9l^2}{\mu} \left(\frac{40}{l} + \frac{36l^2}{\mu^3} \right)+1\right] \left(1-\frac{\mu}{972l} \right)^k P^t_0 \\
		\leq & 754\kappa^4\left(1- \frac{1}{972\kappa} \right)^k\gap_{\hat{f}_t}(x^t_0, y^t_0).
	\end{align*}
	where in the last inequality we note that $P_0^t \leq \frac{11}{10} \gap_{\hat{f}_t}(x^t_0, y^t_0)$. So after $K = O(\kappa\log(\kappa))$ iterations of AGDA, the stopping criterion $\gap_{\hat{f}_t}(x^t_k, y^t_k)\leq \beta \gap_{\hat{f}_t}(x^t_0, y^t_0)$ can be satisfied. 
	
\end{proof}

\begin{remark}
	The theorem above implies that Catalyst-AGDA can achieve the complexity of $\Tilde{O}(\kappa l \Delta \epsilon^{-2})$ in the deterministic setting, which is comparable to the  complexity of Smoothed-AGDA up to a logarithmic term in $\kappa$.

\end{remark}

\newpage

\section{Additional Experiments}
In this section, we show the tuning of Adam, RMSprop and Stochastic AGDA~(SAGDA) for the task of training a toy \textit{regularized} linear WGAN and a toy \textit{regularized} neural WGAN~(one hidden layer). All details on these models are given in the experimental section in the main paper. This section motivates that the smoothed version of stochastic AGDA has superior performance compared to stochastic AGDA that is carefully tuned~(see Figures~\ref{fig:tuning_WGAN_AGDA} and~\ref{fig:tuning_WGAN_ADGA_NN}). Often, the performance is comparable to Adam and RMSprop, if not better~(see Figures~\ref{fig:tuning_WGAN_Adam} and~\ref{fig:tuning_WGAN_Adam_NN}). Findings are similar both for the linear and the neural net cases. We note, as in the main paper, that the stochastic nature of the gradients makes the algorithms converge fast in the beginning and slow down later on.
\begin{figure}[ht]
	\centering
	\includegraphics[width=0.9\textwidth]{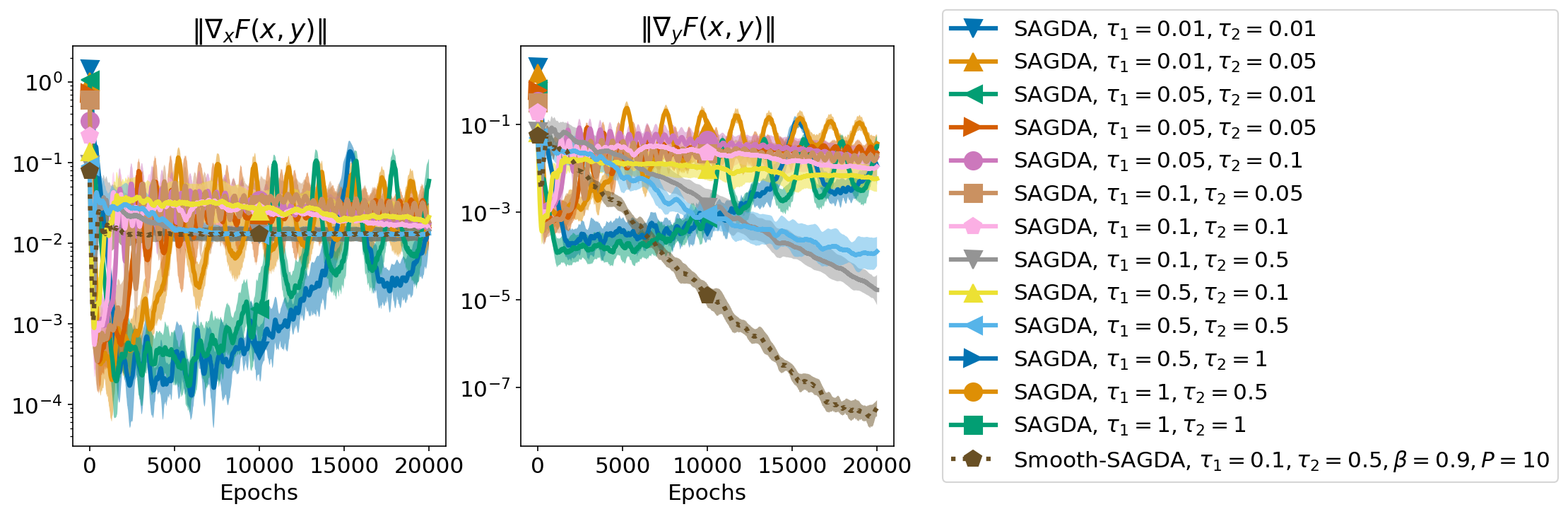}
	\caption{Training of a \textbf{Linear WGAN}~(see experiment section in the main paper for details). Stochastic \textbf{AGDA}~(SAGDA) is compared to the tuned version of Smoothed SAGDA~(best), for different choices of learning rates. Shown is the mean of 3 independent runs and one standard deviation. Smoothing provides acceleration.}
	\label{fig:tuning_WGAN_AGDA}
\end{figure}

\begin{figure}[ht]
	\centering
	\includegraphics[width=0.9\textwidth]{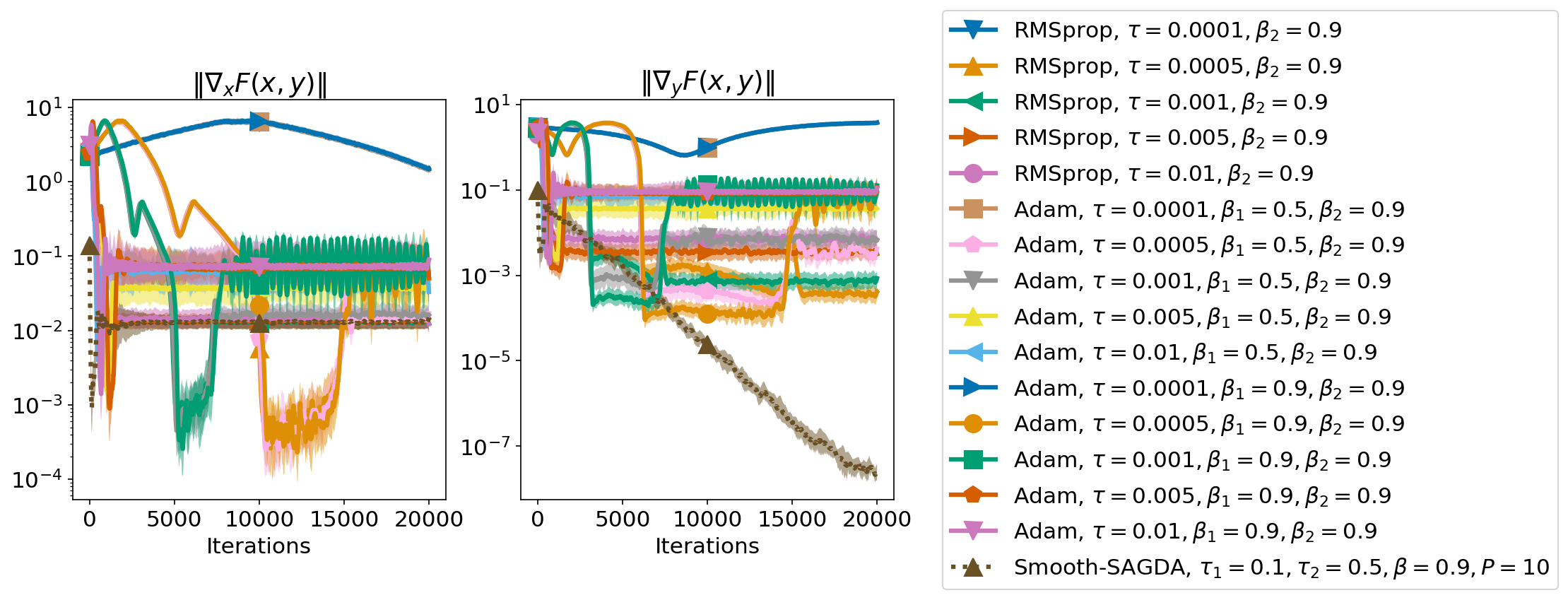}
	\caption{Training of a \textbf{Linear WGAN}~(see experiment section in the main paper for details). \textbf{Adam and RMSprop}~(same learning rate for generator and critic) are compared to the tuned version of Smoothed SAGDA~(best), for different choices hyperparameters. Shown is the mean of 3 independent runs and one standard deviation. Smoothing also in this setting provides acceleration.}
	\label{fig:tuning_WGAN_Adam}
\end{figure}

\begin{figure}[ht]
	\centering
	\includegraphics[width=0.9\textwidth]{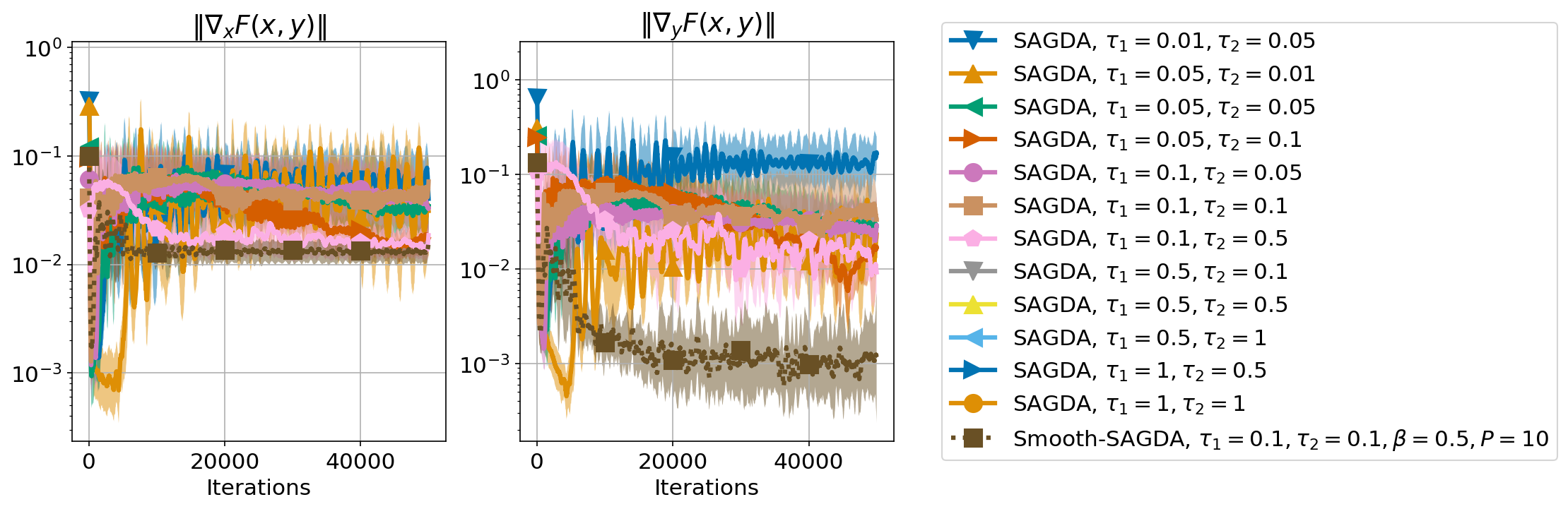}
	\caption{Training of a \textbf{Neural WGAN}~(see experiment section in the main paper for details). Stochastic \textbf{AGDA}~(SAGDA) is compared to the tuned version of Smoothed SAGDA~(best) for different choices of learning rates. Shown is the mean of 3 independent runs and one standard deviation. Smoothing provides acceleration.}
	\label{fig:tuning_WGAN_ADGA_NN}
\end{figure}

\begin{figure}[ht]
	\centering
	\includegraphics[width=0.9\textwidth]{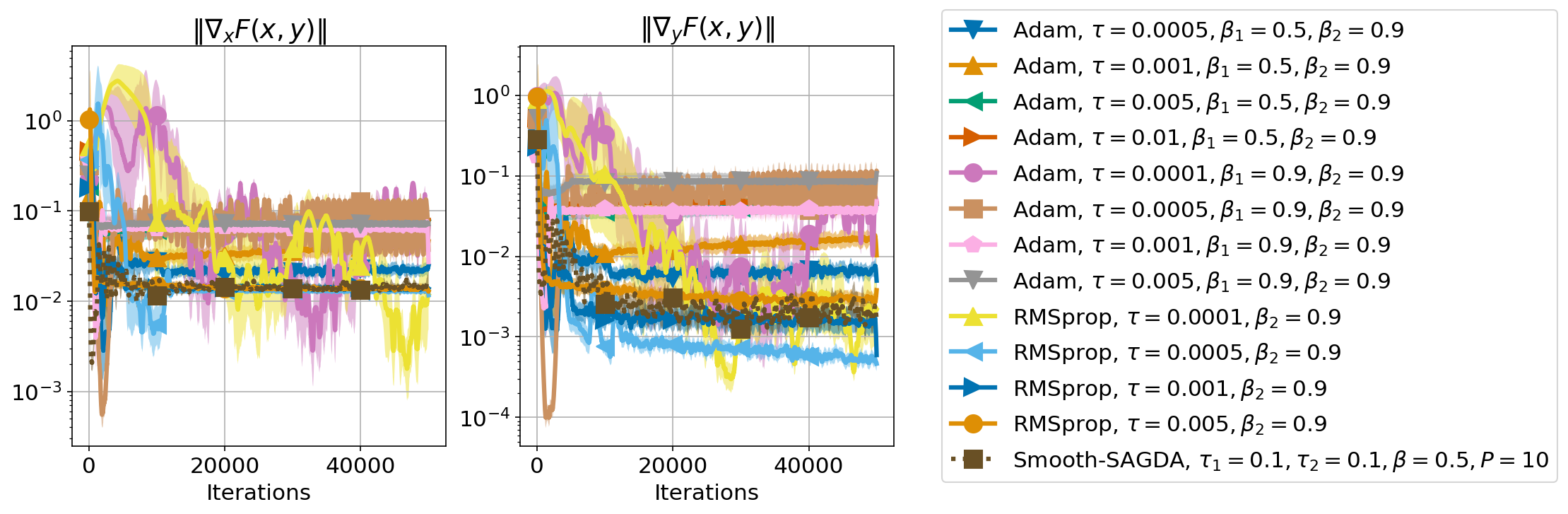}
	\caption{Training of a \textbf{Neural WGAN~}(see experiment section in the main paper for details). \textbf{Adam and RMSprop}~(same learning rate for the generator and critic) are compared to the tuned version of Smoothed SAGDA, for different choices of the hyperparameters. Shown is the mean of 3 independent runs and $1/2$ standard deviation (for better visibility). Performance is slightly worse than RMSprop tuned at best. As mentioned in the main paper, we believe a combination of adaptive stepsizes and smoothing would lead to the best results. 
	}
	\label{fig:tuning_WGAN_Adam_NN}
\end{figure}

\end{document}